\documentclass{article}

     \usepackage[final, nonatbib]{neurips_2019}

\usepackage[utf8]{inputenc} 
\usepackage[T1]{fontenc}    
\usepackage{hyperref}       
\usepackage{url}            
\usepackage{booktabs}       
\usepackage{amsfonts}       
\usepackage{nicefrac}       
\usepackage{microtype}      
\usepackage{subcaption}

\usepackage{etoolbox}
\usepackage{wrapfig}
\usepackage{mathrsfs}
\usepackage{bm}
\usepackage{algorithm}
\usepackage{algcompatible}
\usepackage[noend]{algpseudocode}
\usepackage{mathtools}
\usepackage{float}
\usepackage{times}
\usepackage{multirow,color,graphicx}
\usepackage{subcaption}
\usepackage{caption}
\usepackage{xfrac,amsbsy}
\usepackage[titletoc]{appendix}
\usepackage{xspace}
\usepackage{savesym,verbatim}
\usepackage{paralist}
\usepackage{subscript}
\usepackage{tikz}
\usepackage{verbatim}
\usepackage{amsthm}
\usepackage{bbm}

\newtheorem{theorem}{Theorem}
\newtheorem{proposition}{Proposition}

\usepackage{slashbox,multirow}

\newcommand{\algAgn}{\textsc{Agnostic}\xspace}
\newcommand{\algCon}{\textsc{Conserv}\xspace}

\newcommand{\algAwareCMDP}{\textsc{Aware-CMDP}\xspace}
\newcommand{\algAwareBiLevel}{\textsc{Aware-BiL}\xspace}

\newcommand{\algAdAwareGreedy}{\textsc{AdAware-Vol}\xspace}
\newcommand{\algAdAwareLine}{\textsc{AdAware-Lin}\xspace}
\newcommand{\algLineSearch}{\textsc{LineSearch}\xspace}
\newcommand{\R}{\mathbb{R}}
\newcommand{\E}{\mathbb{E}}

\newcommand{\wopt}{\ensuremath{\mathbf{w}_r^*}}

\DeclareMathOperator*{\argmin}{arg\,min}
\DeclareMathOperator*{\argmax}{arg\,max}
\DeclareMathOperator\diam{diam}

\newcommand{\abs}[1]{\left\vert#1\right\vert}
\def \argmax {\mathop{\rm arg\,max}}
\def \argmin {\mathop{\rm arg\,min}}

\newcommand{\hard}{\ensuremath{\textnormal{hard}}}
\newcommand{\soft}{\ensuremath{\textnormal{soft}}}
\newcommand{\up}{\ensuremath{\textnormal{up}}}
\newcommand{\low}{\ensuremath{\textnormal{low}}}

\makeatletter

\newcommand{\Rmnum}[1]{\expandafter\@slowromancap\romannumeral #1@}
\makeatother

\newcommand{\learner}{\mathsf{L}}
\newcommand{\teacher}{\mathsf{T}}

\title{Learner-aware Teaching: Inverse Reinforcement Learning with Preferences and Constraints}

\author{
    Sebastian Tschiatschek\thanks{Authors contributed equally to this work.}\\
    Microsoft Research\\
    \texttt{setschia@microsoft.com}
    \And
    Ahana Ghosh\footnotemark[1]\\
    MPI-SWS\\
    \texttt{gahana@mpi-sws.org}
    \And
    Luis Haug\footnotemark[1]\\
    ETH Zurich\\
    \texttt{lhaug@inf.ethz.ch}
    \AND
    Rati Devidze\\
    MPI-SWS\\
    \texttt{rdevidze@mpi-sws.org}
    \And
    Adish Singla\\
    MPI-SWS\\
    \texttt{adishs@mpi-sws.org}
}

\begin{document}
\maketitle

\newtoggle{longversion}
\settoggle{longversion}{true}
\begin{abstract}
Inverse reinforcement learning (IRL) enables an agent to learn complex behavior by observing demonstrations from a (near-)optimal policy. The typical assumption is that the learner's goal is to match the teacher’s demonstrated behavior.
In this paper, we consider the setting where the learner has its own preferences that it additionally takes into consideration. These preferences can for example capture behavioral biases, mismatched worldviews, or physical constraints. We study two teaching approaches: \emph{learner-agnostic} teaching, where the teacher provides demonstrations from an optimal policy ignoring the learner's preferences, and \emph{learner-aware} teaching, where the teacher accounts for the learner’s preferences. We design learner-aware teaching algorithms and show that significant performance improvements can be achieved over learner-agnostic teaching.
\end{abstract}
\vspace{-2mm}
\section{Introduction}\label{sec:intro}
Inverse reinforcement learning (IRL) enables a learning agent (\emph{learner}) to acquire skills from observations of a \emph{teacher}'s 
demonstrations. 
The learner infers a
reward function explaining the demonstrated behavior and optimizes its own behavior accordingly. 
IRL has been studied extensively~
\cite{abbeel2004apprenticeship,ratliff2006maximum,ziebart2010modeling,boularias2011relative,osa2018algorithmic} under the premise that the learner can and is willing to imitate the teacher's behavior. 

In real-world settings, however, a learner typically does not blindly follow the teacher's demonstrations, but also has its own preferences and constraints.
For instance, 
consider demonstrating to an auto-pilot of a self-driving car how to navigate from A to B by taking the most fuel-efficient route. These demonstrations might conflict with the preference of the auto-pilot to drive on highways in order to ensure maximum safety.
Similarly, in robot-human interaction with the goal of teaching people how to cook, a teaching robot might demonstrate to a human user how to cook ``roast chicken'', which could conflict with the preferences of the learner who is ``vegetarian''.
To give yet another example, consider a surgical training simulator which provides virtual demonstrations of expert behavior; a novice learner might not be confident enough to imitate a difficult procedure because of safety concerns.
In all these examples, the learner might not be able to acquire useful skills from the teacher's demonstrations.

In this paper, we formalize the problem of teaching a learner with preferences and constraints.
First, we are interested in understanding the suboptimality of \emph{learner-agnostic} teaching, i.e., ignoring the learner's preferences. Second, we are interested in designing \emph{learner-aware} teachers who account for the learner's preferences and thus enable more efficient learning.
To this end, we study a learner model with preferences and constraints in the context of the Maximum Causal Entropy (MCE) IRL framework~\cite{ziebart2010modeling,ziebart2013principle,zhou2018mdce}.
This enables us to formulate the teaching problem as an optimization problem, and to derive and analyze algorithms for learner-aware teaching.
Our main contributions are:
\begin{enumerate}[I]
\item We formalize the problem of IRL under preference constraints (Section~\ref{sec.model} and Section~\ref{sec:learner}).
\item We analyze the problem of optimizing demonstrations for the learner when preferences are  \emph{known} to the teacher, and we propose a bilevel
    optimization approach to the problem (Section~\ref{sec:teacher.fullknowledge}).
\item We propose strategies for adaptively teaching a
    learner with  preferences \emph{unknown} to the teacher, and we provide theoretical guarantees under natural assumptions (Section \ref{sec:teacher.unknown}).
\item We empirically show that significant performance improvements can be achieved by learner-aware teachers as compared to learner-agnostic teachers (Section \ref{sec:experiments}).
\end{enumerate}

\section{Problem Setting}\label{sec.model}
\textbf{Environment.}
Our environment is described by a \emph{Markov decision process} (MDP) $\mathcal{M} := (\mathcal{S},\mathcal{A},T,\gamma,P_0,R)$. Here $\mathcal{S}$ and $\mathcal{A}$ denote finite sets of states and actions. $T\colon \mathcal{S} \times \mathcal{S} \times \mathcal{A} \rightarrow [0,1]$ describes the state transition dynamics, i.e., $T(s' | s,a)$ is the probability of landing in state $s'$ by taking action $a$ from state $s$. $\gamma \in (0,1)$ is the discounting factor. $P_0: \mathcal{S} \rightarrow [0,1]$ is an initial distribution over states. 
$R: \mathcal{S} \rightarrow \R$ is the reward function. We assume that there exists a feature map $\phi_r \colon \mathcal{S} \rightarrow [0, 1]^{d_r}$ such that the reward function is linear, i.e.,
$R(s) = \langle \wopt, \phi_r(s) \rangle$ for some $\wopt \in \R^{d_r}$. Note that a bound of $\Vert \wopt \Vert_1 \leq 1$ ensures that $\abs{R(s)} \leq 1$ for all $s$.


\textbf{Basic definitions.} A \emph{policy} is a map $\pi: \mathcal{S} \times \mathcal{A} \rightarrow [0,1]$ such that $\pi(\;\cdot \mid s)$ is a probability distribution over actions for every state $s$. We denote by $\Pi$ the set of all such policies.  The performance measure for policies we are interested in is the \emph{expected discounted reward} $R(\pi) := \E \left(\sum_{t=0}^\infty \gamma^t R(s_t)\right)$, where the expectation is taken with respect to the distribution over trajectories $\xi = (s_0, s_1, s_2, \ldots)$ induced by $\pi$ together with the transition probabilities $T$ and the initial state distribution $P_0$. A policy $\pi$ is \emph{optimal} for the reward function $R$ if $\pi \in \argmax_{\pi' \in \Pi} R(\pi')$, and we denote an optimal policy by $\pi^*$. 
Note that $R(\pi) = \langle \wopt, \mu_r(\pi) \rangle$, where $\mu_r\colon \Pi \to \R^{d_r}$, $\pi \mapsto \E \left( \sum_{t=0}^\infty\gamma^t \phi_r(s_t)\right)$, is the map taking a policy to its vector of \emph{(discounted) feature expectations}. We denote by $\Omega_r = \{\mu_r(\pi): \pi \in \Pi\}$ the image $\mu_r(\Pi)$ of this map.  Note that the set $\Omega_r \in \R^{d_r}$ is convex (see~\cite[Theorem 2.8]{ziebart2010modeling} and \cite{abbeel2004apprenticeship}), and also bounded due to the discounting factor $\gamma \in (0,1)$.
For a finite collection of trajectories $\Xi = \{s^i_0, s^i_1, s^i_2, \ldots \}_{i=1, 2,\dots}$ obtained by executing a policy $\pi$ in the MDP $\mathcal{M}$, we denote the empirical counterpart of  $\mu_r(\pi)$ by $\hat{\mu}_r(\Xi) := \frac{1}{\abs{\Xi}}\sum_{i}  \sum_{t} {\gamma^t \phi_r(s^{i}_t)}$.

\textbf{An IRL learner and a teacher.} We consider a learner $\learner$ implementing an inverse reinforcement learning (IRL) algorithm and a teacher $\teacher$. The teacher has access to the full MDP $\mathcal{M}$; the learner knows the MDP and the parametric form of reward function $R(s) = \langle \mathbf{w}_r, \phi_r(s) \rangle$ but does not know the true reward parameter $\wopt$. The learner, upon receiving demonstrations from the teacher, outputs a policy $\pi^{\learner}$ using its algorithm.   The teacher's objective is to provide a set of demonstrations $\Xi^{\teacher}$ to the learner that ensures that the learner's output policy $\pi^{\learner}$ achieves high reward $R(\pi^{\learner})$.  

The standard IRL algorithms are based on the idea of \emph{feature matching}~\cite{abbeel2004apprenticeship,ziebart2010modeling,osa2018algorithmic}: The learner's algorithm finds a policy $\pi^{{\learner}}$ that matches the feature expectations of the received demonstrations, ensuring that $\Vert \mu_r(\pi^{{\learner}}) - \hat{\mu}_r(\Xi^{{\teacher}}) \Vert_2 \leq \epsilon$ where $\epsilon$ specifies a desired level of accuracy. In this standard setting, the learner's primary goal is to imitate the teacher (via feature matching) and this makes the teaching process easy.  In fact, the teacher just needs to provide a sufficiently rich pool of demonstrations $\Xi^{{\teacher}}$ obtained by executing $\pi^*$, ensuring $\Vert \hat{\mu}_r(\Xi^{{\teacher}}) - \mu_r(\pi^*) \Vert_2 \leq \epsilon$.  This guarantees that $\Vert \mu_r(\pi^{{\learner}}) - \mu_r(\pi^*) \Vert_2 \leq 2 \epsilon$. Furthermore, the linearity of rewards and $\Vert \wopt \Vert_1 \leq 1$ ensures that the learner's output policy $\pi^{{\learner}}$ satisfies $R(\pi^{{\learner}}) \geq  R(\pi^*) - 2 \epsilon$.


\begin{figure*}[t!]
\centering
	\begin{subfigure}[b]{0.3\textwidth}
	   \centering
		\includegraphics[height=2.4cm]{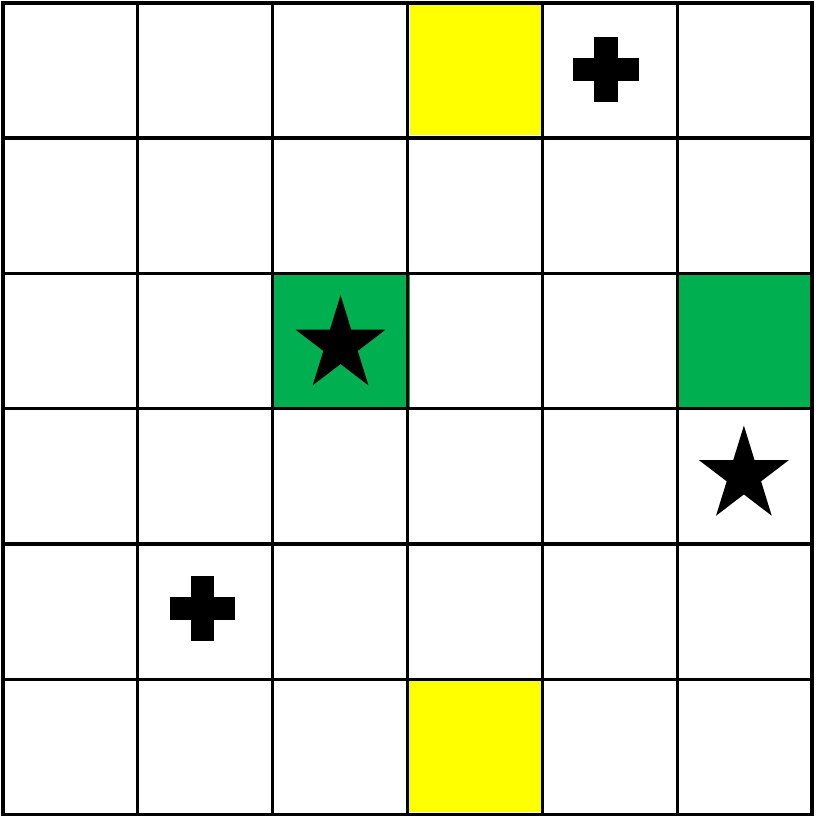}
		\caption{Environment}
		\label{fig:model.example.env}
	\end{subfigure}
	\quad
	\begin{subfigure}[b]{0.36\textwidth}
	    \centering
		\includegraphics[ height=2.4cm]{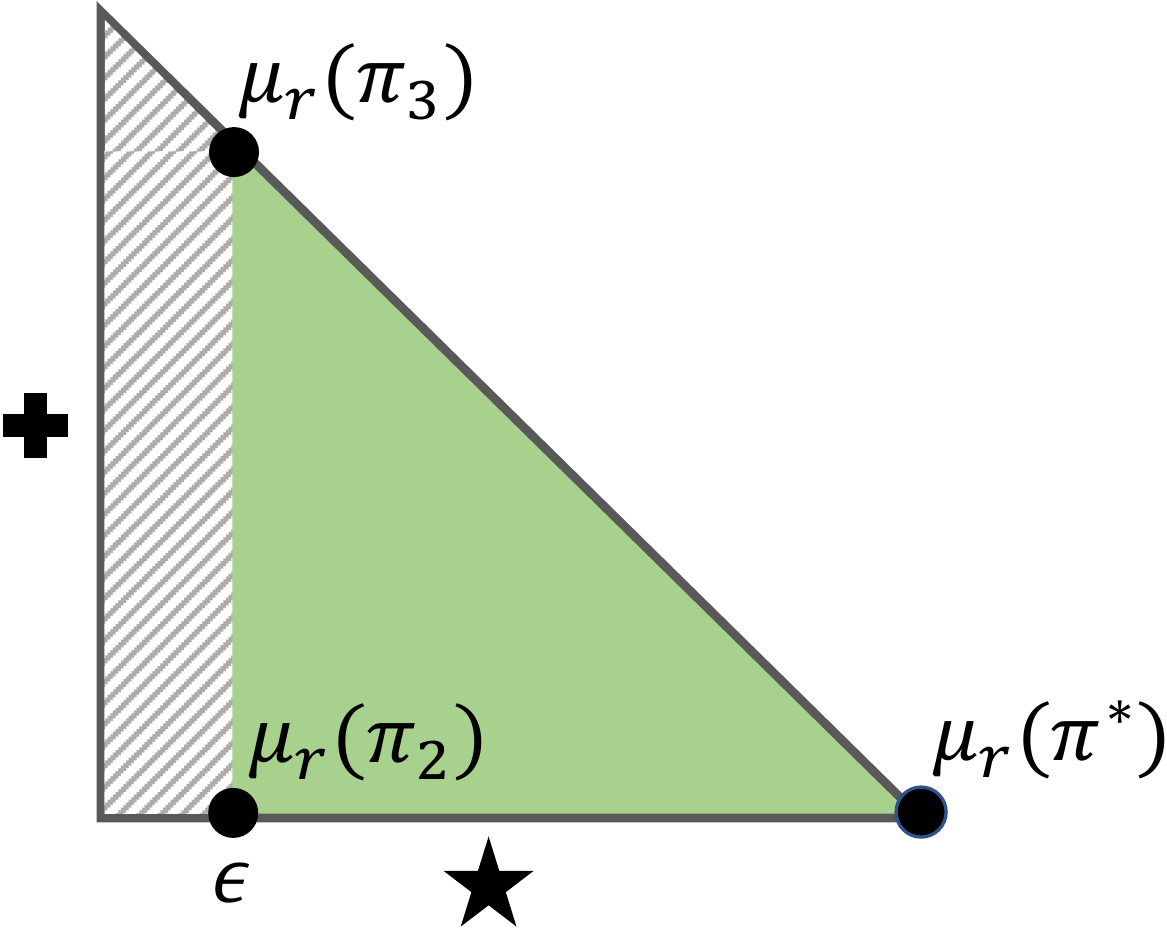}
		\caption{Set of $\mu_r(\pi)$ vectors}
		\label{fig:model.example.mu}
	\end{subfigure}		
   \vspace{-1mm}
   \caption{An illustrative example to showcase the suboptimality of teaching when the learner has preferences and constraints. \textbf{Environment:} Figure~\ref{fig:model.example.env} shows a grid-world environment inspired by the object-world and gathering game environments~\cite{levine2010feature,leibo2017multi,mendez18lifelongIRL}. Each cell represents a state, there are five actions given by {``left", ``up", ``right", "down", ``stay"}, the transitions are deterministic, and the starting state is the top-left cell. The agent's goal is to collect objects in the environment: Collecting a ``star" provides a reward of $1.0$ and a ``plus" a reward of $0.9$; objects immediately appear again upon collection, and the rewards are discounted with $\gamma$ close to 1.  The optimal policy $\pi^*$ is to go to the nearest ``star" and then ``stay" there.  \textbf{Preferences:} A small number of states in the environment are distractors, depicted by colored cells in Figure~\ref{fig:model.example.env}. We consider a learner who prefers to avoid ``green" distractors: it has a hard constraint that the probability of having a ``green" distractor within a $3\textnormal{x}3$ neighborhood, i.e., 1-cell distance, is at most $\epsilon=0.1$. \textbf{Feature expectation vectors:} Figure~\ref{fig:model.example.mu} shows the set of feature expectation vectors $\{\mu_r({\pi}): \pi \in \Pi\}$. The $x$-axis and the $y$-axis represent the discounted feature count for collecting ``star" and ``plus" objects, respectively. The striped region represents policies that are feasible w.r.t.\ the learner's constraint.  \textbf{Suboptimality of teaching:} Upon receiving demonstrations from an optimal policy $\pi^*$ with feature vector $\mu_r({\pi^*})$, the learner under its preference constraint can best match the teacher's demonstrations (in a sense of minimizing $\Vert \mu_r(\pi^{{\learner}}) - \mu_r(\pi^*) \Vert_2$) by outputting a policy with $\mu_r(\pi_2)$, which is clearly suboptimal w.r.t.\ the true rewards. Policy $\pi_3$ with feature vector $\mu_r({\pi_3})$  represents an alternate teaching policy which would have led to higher reward for the learner.}   
	\label{fig:model.example}
    \vspace{-4mm}
\end{figure*}

\textbf{Key challenges in teaching a learner with preference constraints.} In this paper, we study a novel setting where the learner has its own preferences which it additionally takes into consideration when learning a policy $\pi^{\learner}$ using teacher's demonstrations. We formally specify our learner model in the next section; here we highlight the key challenges that arise in teaching such a learner. Given that the learner's primary goal is no longer just imitating the teacher via feature matching, the learner's output policy can be suboptimal with respect to the true reward even if it had access to $\mu_r(\pi^*)$, i.e., the feature expectation vector of an optimal policy $\pi^*$. Figure~\ref{fig:model.example} provides an illustrative example to showcase the suboptimality of teaching when the learner has preferences and constraints. The key challenge that we address in this paper is that of designing a teaching algorithm that selects demonstrations while accounting for the learner's preferences.

\section{Learner Model} \label{sec:learner}
\vspace{-2mm}
In this section we describe the learner models we consider, including different ways of defining preferences and constraints. First, we introduce some notation and definitions that will be helpful.  We capture learner's preferences via a feature map $\phi_c: \mathcal{S} \rightarrow [0, 1]^{d_c}$. We define $\phi(s)$ as a concatenation of the two feature maps $\phi_r(s)$ and $\phi_c(s)$ given by  $[\phi_r(s)^{\dagger}, \phi_c(s)^{\dagger}]^{\dagger}$ and let $d = d_r + d_c$.  Similar to the map $\mu_r$, we define $\mu_c \colon \Pi \to \R^{d_c}$, $\pi \mapsto \E \left( \sum_{t=0}^\infty \gamma^t \phi_c(s_t)\right)$ and $\mu \colon \Pi \to \R^{d}$, $\pi \mapsto \E \left( \sum_{t=0}^\infty \gamma^t \phi(s_t)\right)$. Similar to $\Omega_r$, we define $\Omega_c \subseteq \R^{d_c}$ and  $\Omega \subseteq \R^{d}$ as the images of the maps $\mu_c(\Pi)$ and $\mu(\Pi)$. Note that for any policy $\pi \in \Pi$, we have $\mu(\pi) = [\mu_r(\pi)^{\dagger}, \mu_c(\pi)^{\dagger}]^{\dagger}$.
 
\textbf{Standard (discounted) MCE-IRL.}
Our learner models build on the (discounted) Maximum Causal Entropy (MCE) IRL framework~\cite{ziebart2008maximum,ziebart2010modeling,ziebart2013principle,zhou2018mdce}. In the standard (discounted) MCE-IRL framework, a learning agent aims to identify a policy that  matches the feature expectations of the teacher's demonstrations while simultaneously maximizing the (discounted) causal entropy given by $H(\pi) := H(\{a_t\}_{t=0, 1, \ldots} \Vert \{s_t\}_{t=0, 1, \ldots}) := \sum_{t=0}^{\infty} \gamma^t \E\Big[-\log \pi(a_t \mid s_t) \Big]$. 
\iftoggle{longversion}{%
More background is provided in Appendix~\ref{appendix:background-mceirl}.%
}
{%
More background is provided in Appendix D of the supplementary.%
}
 
\textbf{Including preference constraints.} 
The standard framework can be readily extended to include learner's preferences in the form of constraints on the preference features $\phi_c$. Clearly, the learner's preferences can render exact matching of the teacher's demonstrations infeasible and hence we relax this condition. 
To this end, we consider the following generic learner model:
\begin{align}
\max_{\pi,\ \delta^{\soft}_{r} \geq 0,\ \delta^{\soft}_{c} \geq 0} \quad &H(\pi) - C_r  \cdot \Vert \delta^{\soft}_{r} \Vert_p - C_c  \cdot \Vert  \delta^{\soft}_{c} \Vert_p \label{sec3.eq.generic-learner} \\
\textnormal{s.t.\quad}
&\lvert \mu_{r}(\pi)[i] - \hat{\mu}_{r}(\Xi^{\teacher})[i] \rvert \leq \delta^{\hard}_{r}[i] + \delta^{\soft}_{r}[i] \ \forall i \in \{1, 2, \ldots, d_r\}    \notag \\
&\qquad \qquad \quad  g_j(\mu_{c}(\pi)) \leq \delta^{\hard}_{c}[j] + \delta^{\soft}_{c}[j] \ \forall j \in \{1, 2, \ldots, m\}, \notag
\end{align}
Here, $g\colon \R^{d_c} \mapsto \R$ are $m$ convex functions representing preference constraints. The coefficients $C_r$ and $C_c$ are the learner's parameters which quantify the relative importance of matching the teacher's demonstrations and satisfying the learner's preferences.  The learner model is further characterized by parameters $\delta^{\hard}_{r}[i]$ and $\delta^{\hard}_{c}[j]$ (we will use the vector notation as $\delta^{\hard}_{r} \in \R^{d_r}_{\geq 0}$
and $\delta^{\hard}_{c} \in \R^{m}_{\geq 0}$). 
The optimization variables for the learner are given by $\pi$, $\delta^{\soft}_{r}[i]$, and $\delta^{\soft}_{c}[j]$ (we will use the vector notation as $\delta^{\soft}_{r} \in \R^{d_r}_{\geq 0}$
and $\delta^{\soft}_{c} \in \R^{m}_{\geq 0}$). These parameters ($\delta_r^{\hard}$, $\delta_c^{\hard}$) and optimization variables ($\delta_r^{\soft}$, $\delta_c^{\soft}$) characterize the following behavior:
\begin{itemize}
    \item While a mismatch of up to $\delta_r^{\hard}$ between the learner's and teacher's reward feature expectations incurs no cost regarding the optimization objective, a  mismatch larger than $\delta_r^{\hard}$ incurs a cost of $C_r \cdot \| \delta_r^{\soft} \|_p$.
    \item Similarly, while a violation of up to $\delta_c^{\hard}$ of the learner's preference constraints incurs no cost regarding the optimization objective, a  violation larger than $\delta_c^{\hard}$ incurs a cost of $C_c \cdot \| \delta_c^{\soft} \|_p$.
\end{itemize}
Next, we discuss two special instances of this generic learner model.

\subsection{Learner Model with Hard Preference Constraints}\label{sec:learner:hard}

It is instructive to study a special case of the above-mentioned generic learner model. Let us consider the model in Eq.~\ref{sec3.eq.generic-learner} with $\delta^{\hard}_{r}=0, \delta^{\hard}_{c}=0$, and a limiting case with $C_r, C_c \gg 0$ such that the term $H(\pi)$ can be neglected.  Now, if we additionally assume that $C_c \gg C_r$, the learner's objective can be thought of as finding a policy $\pi$ that minimizes the $L^p$ norm distance to the reward feature expectations of the teacher's demonstration while satisfying the constraints $g_j(\mu_{c}(\pi)) \leq 0 \ \forall j \in \{1, 2, \ldots, m\}$.
More formally, we study the following learner model given in Eq.~\ref{sec3.1.eq.generic-learner--limitingcase} below:
\begin{align}
\min_{\pi} \quad & \Vert \mu_{r}(\pi) - \hat{\mu}_{r}(\Xi^{\teacher}) \Vert_p  \label{sec3.1.eq.generic-learner--limitingcase} \\
\textnormal{s.t.\quad}
& g_j(\mu_{c}(\pi)) \leq 0 \ \forall j \in \{1, 2, \ldots, m\}. \notag
\end{align}
%
%
%
To get a better understanding of the model, we can define the learner's constraint set as $\Omega^{\learner} := \{\mu: \mu \in \Omega \textnormal{ s.t. } g_j(\mu_{c}) \leq 0 \ \forall j \in \{1, 2, \ldots, m\} \}$. Similar to $\Omega^{\learner}$, we define $\Omega^{\learner}_r \subseteq \Omega_r$ where $\Omega^{\learner}_r$ is the projection of the set $\Omega^{\learner}$ to the subspaces $\R^{d_r}$.  We can now rewrite the above optimization problem as $\min_{\pi\colon \mu_r(\pi) \in \Omega^{\learner}_r} \Vert \mu_{r}(\pi) - \hat{\mu}_{r}(\Xi^{\teacher}) \Vert_p$.
%
Hence, the learner's behavior is given by: 
%
\begin{enumerate}[(i)]
\item  \emph{Learner can match:} When $\hat{\mu}_r(\Xi^{\teacher}) \in \Omega_r^{\learner}$, the learner outputs a policy $\pi^{\learner}$ s.t.\ $\mu_r(\pi^{\learner}) = \hat{\mu}_{r}(\Xi^{\teacher})$.
\item \emph{Learner cannot match:} Otherwise, the learner outputs a policy $\pi^{\learner}$ such that $\mu_r(\pi^{\learner})$ is given by the $L^p$ norm projection of the vector $\hat{\mu}_{r}(\Xi^{\teacher})$ onto the set $\Omega_r^{\learner}$.
\end{enumerate}
Figure~\ref{fig:model.example} provides an illustration of the behavior of this learner model. We will design learner-aware teaching algorithms for this learner model in Section~\ref{sec:teacher.fullknowledge.special} and Section~\ref{sec:teacher.unknown}.

   
\subsection{Learner Model with Soft Preference Constraints}\label{sec:learner:soft}
Another interesting learner model that we study in this paper arises  from the generic learner when we consider $m=d_c$ number of box-type linear constraints with $g_j(\mu_c(\pi)) = \mu_c(\pi)[j] \ \forall j \in \{1, 2, \ldots, d_c\}$. We consider an $L^1$ norm penalty on violation, and for simplicity we consider $\delta^{\hard}_{r}[i]=0 \ \forall i \in \{1, 2, \ldots, d_r\}$. 
In this case, the learner's model is given by
\begin{align}
\max_{\pi,\ \delta^{\soft}_{r} \geq 0,\ \delta^{\soft}_{c} \geq 0} \quad &H(\pi) - C_r  \cdot \Vert \delta^{\soft}_{r} \Vert_1 - C_c  \cdot \Vert  \delta^{\soft}_{c} \Vert_1 \label{sec3.eq.soft-learner} \\
\textnormal{s.t.\quad}
&\lvert \mu_{r}(\pi)[i] - \hat{\mu}_{r}(\Xi^{\teacher})[i] \rvert \leq \delta^{\soft}_{r}[i] \ \forall i \in \{1, 2, \ldots, d_r\}    \notag \\
&\qquad \quad \quad \ \ \ \ \ \mu_c(\pi)[j] \leq \delta^{\hard}_{c}[j] + \delta^{\soft}_{c}[j] \ \forall j \in \{1, 2, \ldots, d_c\}, \notag
\end{align}


The solution to the above problem corresponds to a \emph{softmax} policy with a  reward function $R_{\bm{\lambda}}(s) = \langle \bm{w}_{\bm{\lambda}}, \phi(s) \rangle$ where $\bm{w}_{\bm{\lambda}} \in \R^d$ is parametrized by $\bm{\lambda}$. The optimal parameters $\bm{\lambda}$ can be computed efficiently and the corresponding softmax policy is then obtained by \emph{Soft-Value-Iteration} procedure (see  \cite[Algorithm.~9.1]{ziebart2010modeling}, \cite{zhou2018mdce}). 
\iftoggle{longversion}{%
Details are provided in Appendix~\ref{appendix:mceirl-with-preferences}.%
}
{%
Details are provided in Appendix E of the supplementary.%
}
We will design learner-aware teaching algorithms for this learner model in Section~\ref{sec:teacher.fullknowledge.generic}.

\section{Learner-aware Teaching under Known Constraints}\label{sec:teacher.fullknowledge}
In this section, we analyze the setting when the teacher has full knowledge of the learner's constraints. 

\subsection{A Learner-aware Teacher for Hard Preferences: \algAwareCMDP}\label{sec:teacher.fullknowledge.special}
Here, we design a learner-aware teaching algorithm when considering the learner from Section~\ref{sec:learner:hard}. Given that the teacher has full knowledge of the learner's preferences, it can compute an optimal teaching policy by maximizing the reward over policies that satisfy the learner's preference constraints, i.e., the teacher solves a constrained-MDP problem (see \cite{de1960problemes,altman1999constrained}) given by 
\begin{align*}
\max_{\pi} \quad & \langle \wopt, \mu_r(\pi) \rangle \quad 
\textnormal{s.t.} \quad
\mu_r(\pi) \in \Omega^\learner_r.
\end{align*}
We refer to an optimal solution of this problem as $\pi^\textnormal{aware}$ and the corresponding teacher as \algAwareCMDP.  We can make the following observation formalizing the value of  learner-aware teaching:
\begin{theorem}
For simplicity, assume that the teacher can provide an exact feature expectation $\mu(\pi)$ of a policy instead of providing demonstrations to the learner. Then, the value of learner-aware teaching is
\begin{align*}
\max_{\pi \textnormal{ s.t. } \mu_r(\pi) \in \Omega^\learner_r} \Big\langle \wopt, \mu_r(\pi) \Big\rangle - \Big\langle \wopt, \textnormal{Proj}_{\Omega^\learner_r}\big(\mu_r(\pi^*)\big) \Big\rangle \geq 0.
\end{align*} 
\end{theorem}

When the set $\Omega^\learner$ is defined via a set of linear constraints, the above problem can be formulated as a linear program and solved exactly.
\iftoggle{longversion}{%
Details are provided in Appendix~\ref{appendix:lp}.%
}
{%
Details are provided in Appendix F the supplementary material.%
}

\subsection{A Learner-aware Teacher for Soft Preferences: \algAwareBiLevel}\label{sec:teacher.fullknowledge.generic}
For the learner models in Section~\ref{sec:learner}, the optimal learner-aware teaching problem can be naturally formalized as the following bi-level optimization problem:
\begin{align}
    \max_{\pi^\teacher} \quad & R(\pi^\learner) \quad
    \textnormal{s.t.} 
       \quad \pi^\learner \in \arg \max_{\pi} \textnormal{IRL}(\pi, \mu(\pi^\teacher)),
\end{align}
where $\textnormal{IRL}(\pi, \mu(\pi^\teacher))$ stands for the IRL problem solved by the learner given demonstrations from $\pi^\teacher$ and can include preferences of the learner (see Eq.~\ref{sec3.eq.generic-learner} in Section~\ref{sec:learner}).

There are many possibilities for solving this bi-level optimization problem---see for example~\cite{sinha2018review} for an overview.
In this paper we adopted a \emph{single-level reduction} approach to simplify the above bi-level optimization problem as this results in particularly intuitive optimiziation problems for the teacher.
The basic idea of single-level reduction is to replace the lower-level problem, i.e., $\arg \max_{\pi} \textnormal{IRL}(\pi, \mu(\pi^\teacher))$, by the optimality conditions for that problem given by the Karush-Kuhn-Tucker conditions \cite{boyd2004convex,sinha2018review}. For the learner model outlined in Section~\ref{sec:learner:soft}, these reductions take the following form 
\iftoggle{longversion}{%
(see Appendix~\ref{appendix:bi-level} for details):%
}
{%
(see Appendix G in the supplementary material for details):%
}
\begin{align}
\max_{\bm{\lambda} := \{\bm{\alpha}^{\low} \in \R^{d_r}, \ \bm{\alpha}^{\up} \in \R^{d_r}, \ \bm{\beta}  \in \R^{d_c}\}} \quad &\langle \wopt, \mu_r(\pi_{\bm{\lambda}}) \rangle \label{sec4.eq.bilevel-teacher} \\
\textnormal{s.t.\quad} 
&0 \leq \bm{\alpha}^{\low} \leq  C_r \notag \\
&0 \leq \bm{\alpha}^{\up} \leq  C_r  \notag \\
\{0 \leq \bm{\beta} \leq  C_c \texttt{\ AND\ } \mu_c(\pi_{\bm{\lambda}}) \leq  \delta_c^{\hard}\}  &\texttt{\ OR\ } \{\bm{\beta} =  C_c \texttt{\ AND\ } \mu_c(\pi_{\bm{\lambda}}) \geq  \delta_c^{\hard}\} \notag
\end{align}
where $\pi_{\bm{\lambda}}$ corresponds to a \emph{softmax} policy with a reward function $R_{\bm{\lambda}}(s) = \langle \bm{w}_{\bm{\lambda}}, \phi(s) \rangle$ for  $\bm{w}_{\bm{\lambda}} = [(\bm{\alpha}^{\low}-\bm{\alpha}^{\up})^\dagger, -\bm{\beta}^\dagger]^\dagger$. Thus, finding optimal demonstrations means optimization over \emph{softmax} teaching policies while respecting the learner's preferences.
To actually solve the above optimization problem and find good teaching policies, we use an approach inspired by the Frank-Wolfe algorithm~\cite{jaggi2013}
\iftoggle{longversion}{
detailed in Appendix~\ref{appendix:bi-level}.
}
{
detailed in Appendix G of the supplementary material. 
}
We refer to a teacher implementing this approach as \algAwareBiLevel.

\section{Learner-Aware Teaching Under Unknown Constraints}
\label{sec:teacher.unknown}
In this section, we consider the more realistic and challenging
setting in which the teacher $\teacher$ does \emph{not} know  the learner $\learner$'s
constraint set $\Omega_r^{\learner}$. Without feedback from $\learner$, $\teacher$ can generally not do better than the agnostic
teacher who simply ignores any constraints. We therefore assume that
$\teacher$ and $\learner$ interact in rounds as described by Algorithm
\ref{algo:interaction1}. The two versions of the algorithm we describe
in Sections \ref{sec:teacher.unknown.adaptive-greedy} and
\ref{sec:teacher.unknown.adaptive-blackbox} are obtained by specifying how
$\teacher$ adapts the teaching policy in each round.

\begin{algorithm}[H]
    \caption{Teacher-learner interaction in the adaptive teaching setting}
    \label{algo:interaction1}
    \begin{algorithmic}[1]
        \State Initial teaching policy $\pi^{\teacher, 0}$ (e.g., optimal policy ignoring any constraints)
        \For{round $i = 0, 1, 2, \dots$}
            \State Teacher provides demonstrations with feature vector $\mu_r^{\teacher,i}$ using policy $\pi^{\teacher, i}$
            \State Learner upon receiving $\mu_r^{\teacher,i}$ computes a policy $\pi^{\learner,i}$ with feature vector $\mu_r^{\learner,i}$
            \State Teacher observes learner's feature vector $\mu_r^{\learner,i}$ and adapts the teaching policy
        \EndFor{}
    \end{algorithmic}
\end{algorithm}

In this section, we assume that 
$\learner$ is as described in Section \ref{sec:learner:hard}: Given demonstrations $\Xi^\teacher$, $\learner$ finds a policy $\pi^{\learner}$ such that $\mu_r(\pi^{\learner})$ matches the $L^2$-projection of $\hat\mu_r(\Xi^{\teacher})$ onto
$\Omega^{\learner}_r$. For the sake of simplifying the presentation
and the analysis, we also assume that $\learner$ and $\teacher$ can
observe the exact feature expectations of their respective
policies, e.g., $\hat \mu_r(\Xi^{\teacher}) = \mu_r(\pi^{\teacher})$ if $\Xi^{\teacher}$ is sampled from $\pi^{\teacher}$.



\subsection{An Adaptive Learner-aware Teacher Using Volume Search: \algAdAwareGreedy}
\label{sec:teacher.unknown.adaptive-greedy}
In our first adaptive teaching algorithm \algAdAwareGreedy, $\teacher$
maintains an estimate
$\hat \Omega_r^{\learner} \supset \Omega_r^{\learner}$ of the
learner's constraint set, which in each round gets updated by
intersecting the current version with a certain affine halfspace, thus
reducing the volume of $\hat \Omega_r^{\learner}$. The new teaching policy is then any policy $\pi^{\teacher, i+1}$ which
is optimal under the constraint that
$\mu^{\teacher, i+1} \in \hat \Omega_r^{\learner}$. The interaction ends as soon as
$\lVert \mu^{\learner,i}_{r} - \mu^{\teacher,i}_{r} \rVert_2 \leq \epsilon$ for a threshold $\epsilon$.
\iftoggle{longversion}{%
Details are provided in Appendix~\ref{sec:appendix.teacher.unknown.greedy}.%
}
{%
Details are provided in Appendix C.1 of the supplementary.%
}

\begin{theorem}
    \label{thm:teacher.unknown.adaptive-greedy}
    Upon termination of \algAdAwareGreedy, $\learner$'s output policy
    $\pi^{\learner}$ satisfies $R(\pi^{\learner}) \geq R(\pi^{\textnormal{aware}}) -
    \epsilon$ for any policy $\pi^\textnormal{aware}$ which is optimal under $\learner$'s constraints.
    For the special case that $\Omega^{\learner}_r$ is a polytope defined by $m$
    linear inequalities, the algorithm terminates in $O(m^{d_r})$
    iterations.
\end{theorem}



\subsection{An Adaptive Learner-aware Teacher Using Line Search:
    \algAdAwareLine}
\label{sec:teacher.unknown.adaptive-blackbox}
In our second adaptive teaching algorithm, \algAdAwareLine, $\teacher$
adapts the teaching policy by performing a binary search on a line
segment of the form
$\{\mu_r^{\learner, i} + \alpha \wopt ~|~ \alpha \in [\alpha_{\min},
\alpha_{\max}]\} \subset \R^{d_r}$ to find a vector
$\mu_r^{\teacher, i+1} = \mu_r^{\learner, i} + \alpha_i \wopt$ that is
the vector of feature expectations of a policy; here $\alpha_{\max} > \alpha_{\min} > 0$ are fixed constants.
If that is not
successful, the teacher finds a teaching policy with
$\mu_r^{\teacher, i+1} \in \argmin_{\mu_r \in \Omega_r} \Vert \mu_r -
\mu_r^{\learner, i} - \alpha_{\min} \wopt \Vert_2$. The following
theorem analyzes the convergence of $\learner$'s performance to $\overline R_{\learner} :
= \max_{\mu_r \in \Omega_r} R(\mu_r)$ under the assumption
that $\teacher$'s search succeeds in every round. 
\iftoggle{longversion}{%
The proof and further details are provided in Appendix~\ref{appendix:teacher.unknown.linesearch}.%
}
{%
The proof and further details are provided in Appendix C.2 of the supplementary.%
}



\begin{theorem}
    \label{thm:convergence}
    Fix some $\varepsilon> 0$ and assume that there exists a constant
    $\alpha_{\min} > 0$ such that, as long as
    $\overline R_\learner - R(\mu_r^{\learner, i}) > \varepsilon$, the
    teacher can find a teaching policy $\pi^{\teacher, i+1}$
    satisfying
    $\mu_r^{\teacher, i+1} = \mu_r^{\learner, i} + \alpha_i \wopt$ for
    some $\alpha_i \geq \alpha_{\min}$. Then the learner's performance
    increases monotonically in each round of \algAdAwareLine, i.e.,
    $R(\mu_r^{\learner, i + 1}) > R(\mu_r^{\learner, i})$. 
    Moreover, after at most
    $O(\frac{D^2}{\varepsilon \alpha_{\min}} \log
    \frac{D}{\varepsilon})$ teaching steps, the learner's performance
    satisfies
    $R(\mu_r^{\learner, i}) > \overline R_\learner -
    2\varepsilon$. Here we abbreviate $D := \diam \Omega_r$.
\end{theorem}


\section{Experimental Evaluation}\label{sec:experiments}

In this section we evaluate our teaching algorithms for different types of learners on the environment introduced in Figure~\ref{fig:model.example}. 
%
The environment we consider here has three types of reward objects, i.e., a ``star" object with  reward of $1.0$, a ``plus" object with reward of $0.9$, and a ``dot" object with reward of $0.2$. Two objects of each type are placed randomly on the grid such that there is always only a single object in each grid cell. 
The presence of an object of type ``star'', ``plus'', or ``dot'' in some state $s$ is encoded in the reward features $\phi_r(s)$ by a binary-indicator for each type such that $d_r=3$.
We use a discount factor of $\gamma=0.99$. Upon collecting an object, there is a $0.1$ probability of transiting to a terminal state.

{\bfseries Learner models.}
We consider a total of 5 different learners whose preferences can be described by \emph{distractors} in the environment. Each learner prefers to avoid a certain subset of these distractors. There is a total of 4 of distractors: (i) two ``green" distractors are randomly placed at a distance of 0-cell and 1-cell to the ``star" objects, respectively; (ii) two ``yellow" distractors are randomly placed at a distance of 1-cell and 2-cells to the ``plus" objects, respectively, see\ Figure~\ref{fig:exp-known-learner-preferences.a}.


Through these distractors we define learners L1-L5 as follows: \textbf{(L1)} no preference features ($d_c=0$); \textbf{(L2)} two preference features ($d_c=2$) such that $\phi_c(s)[1]$ and $\phi_c(s)[2]$ are binary indicators of whether there is a ``green" distractor at a distance of 0-cells or 1-cell, respectively; \textbf{(L3)} four preference features ($d_c=4$) such that $\phi_c(s)[1],\phi_c(s)[2]$ are as for L2, and $\phi_c(s)[3]$ and $\phi_c(s)[4]$ are binary indicators of whether there is a ``green" distractor at a distance of 2-cells or a ``yellow'' distractor at a distance of 0-cells, respectively; \textbf{(L4)} five preference features ($d_c=5$) such that $\phi_c(s)[1],\ldots,\phi_c(s)[4]$ are as for L3, and $\phi_c(s)[5]$ is a binary indicator whether there is a ``yellow'' distractor at a distance of 1-cell; and \textbf{(L5)} six preference features ($d_c=6$)  such that $\phi_c(s)[1],\ldots,\phi_c(s)[5]$ are as for L4, and $\phi_c(s)[6]$ is a binary indicator whether there is a ``yellow'' distractor at a distance of 2-cells.




 
The first row in Figure~\ref{fig:exp-known-learner-preferences} shows an instance of the considered object-worlds and indicates the preference of the learners to avoid certain regions by the gray area.

\begin{figure}[!htbp]
    \centering
    \begin{subfigure}[t]{0.82\textwidth}
      \centering
      \includegraphics[width=0.18\textwidth]{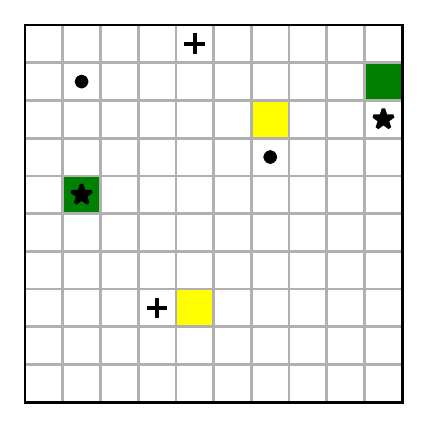}
      \includegraphics[width=0.18\textwidth]{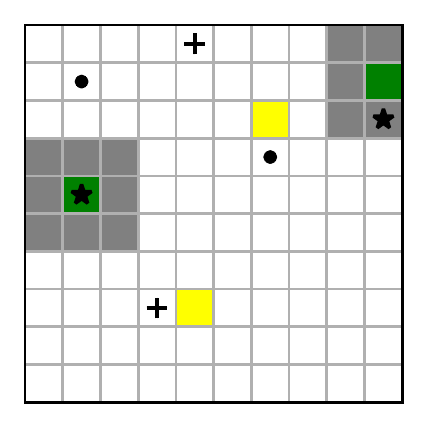}
      \includegraphics[width=0.18\textwidth]{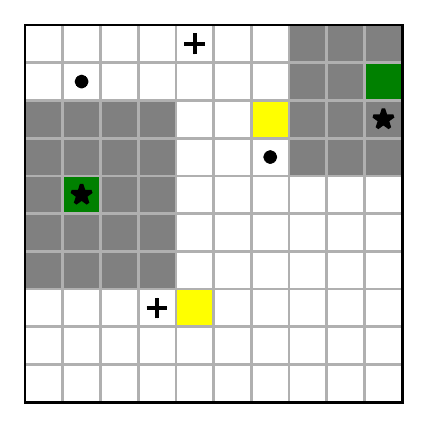}
      \includegraphics[width=0.18\textwidth]{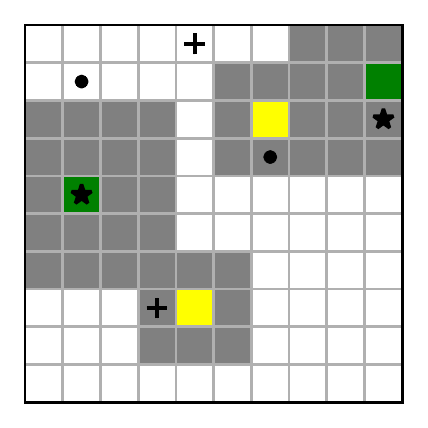}
      \includegraphics[width=0.18\textwidth]{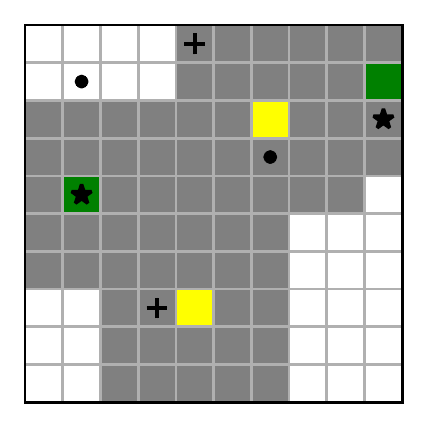}\vspace{-2mm}
      \caption{Environments and learners' preferences for 5 different learners \textsc{L1}, $\ldots$, \textsc{L5}}
      \label{fig:exp-known-learner-preferences.a}
    \end{subfigure}%
    \\
    \begin{subfigure}[t]{0.82\textwidth}
      \centering
      \includegraphics[width=0.18\textwidth]{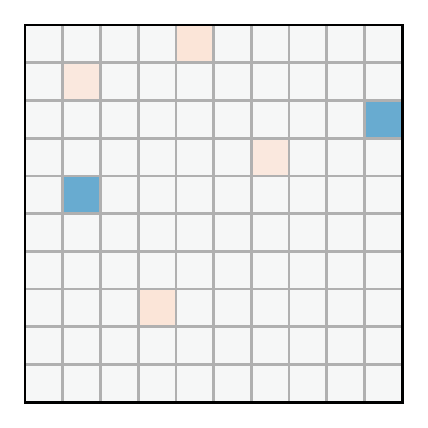}
      \includegraphics[width=0.18\textwidth]{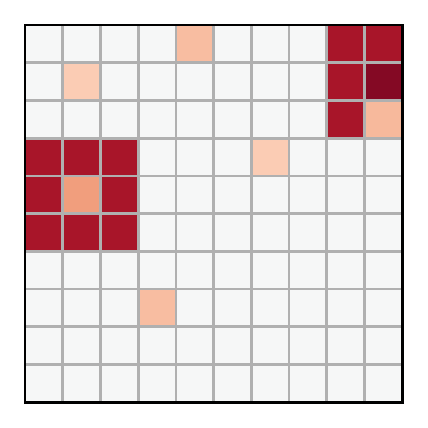}
      \includegraphics[width=0.18\textwidth]{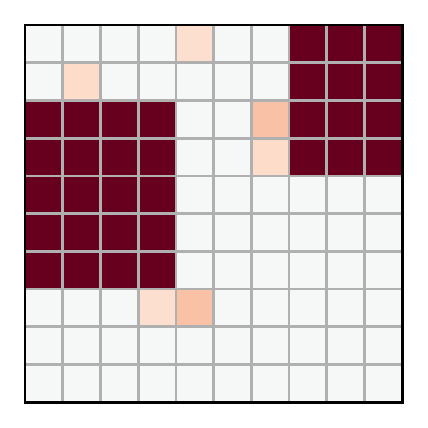}
      \includegraphics[width=0.18\textwidth]{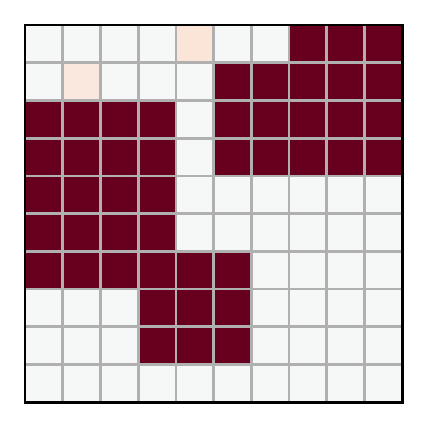}
      \includegraphics[width=0.18\textwidth]{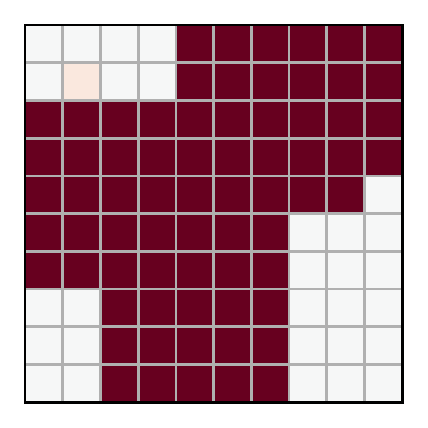}\vspace{-2mm}
      \caption{Learners' rewards inferred from learner-agnostic teacher's (\algAgn)
       demonstrations}
      \label{fig:exp-known-learner-preferences.b}       
    \end{subfigure}%
    \\
    \begin{subfigure}[t]{0.82\textwidth}
      \centering
      \includegraphics[width=0.18\textwidth]{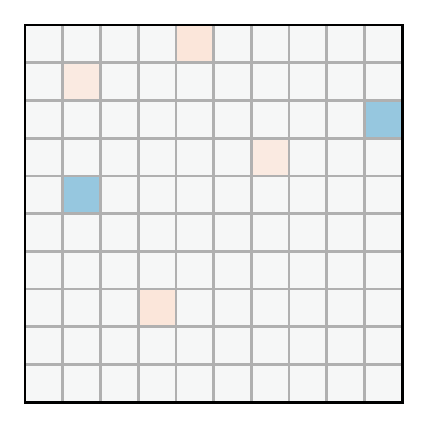}
      \includegraphics[width=0.18\textwidth]{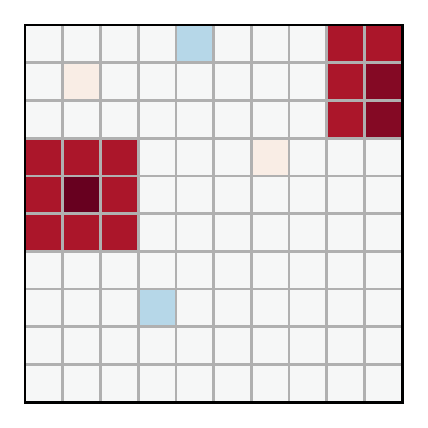}
      \includegraphics[width=0.18\textwidth]{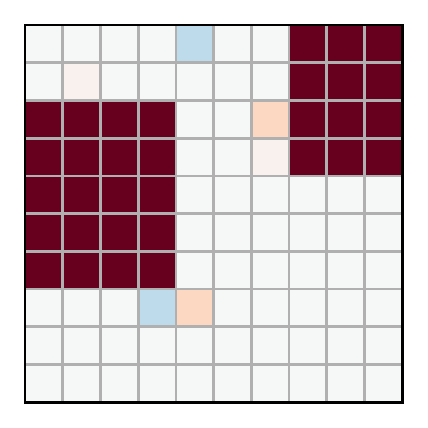}
      \includegraphics[width=0.18\textwidth]{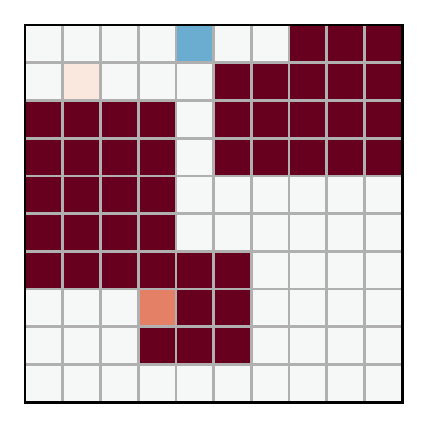}
      \includegraphics[width=0.18\textwidth]{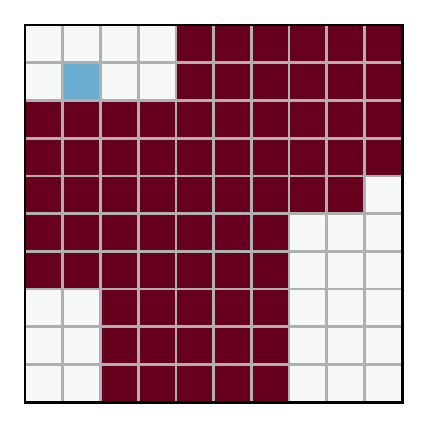}\vspace{-2mm}
      \caption{Learners' rewards inferred from learner-aware teacher's (\algAwareBiLevel) demonstrations}
      \label{fig:exp-known-learner-preferences.c}      
    \end{subfigure}%
    \caption{Teaching in object-world environments under full knowledge of the learner's preferences. Green and yellow cells indicate distractors associated with either ``star" or ``plus" objects, respectively. Learner's preferences to avoid cells are indicated in gray. Learner model from Section~\ref{sec:learner:soft} with $C_r=5$, $C_c=10$, and $\delta^{\hard}_c = 0$ is considered for these experiments. 
    The learner-aware teacher enable the learner to infer reward functions that are compatible with the learner's preferences and achieve higher average rewards.  In  Figure~\ref{fig:exp-known-learner-preferences.b} and Figure~\ref{fig:exp-known-learner-preferences.c}, blue color represents positive reward, red color represents negative reward, and the magnitude of the reward is indicated by color intensity.
    }
    \label{fig:exp-known-learner-preferences}
    \vspace{-3mm}
\end{figure}

\subsection{Teaching under known constraints}\label{sec:experiments.known}


\looseness-1
In this section we consider learners with soft constraints from Section~\ref{sec:learner:soft}, with preference features as described above, and parameters $C_r = 5$, $C_c=10$, and $\delta^{\hard}_c = 0$ (more experimental results for different values of $C_r$ and $C_c$ are provided in 
\iftoggle{longversion}{%
Appendix~\ref{appendix:experiments.known}).%
}
{%
Appendix B.1 of the supplementary).%
}
Our first results are presented in Figure~\ref{fig:exp-known-learner-preferences}.  The second and third rows show the rewards inferred by the learners for demonstrations provided by a learner-agnostic teacher who ignores any constraints (\algAgn) and the bi-level learner-aware teacher (\algAwareBiLevel), respectively.
We observe that \algAgn fails to teach the learner about objects' positive rewards in cases where the learners' preferences conflict with the position of the most rewarding objects (second row).
In contrast, \algAwareBiLevel always successfully teaches the learners about rewarding objects that are compatible with the learners' preferences (third row).

We also compare \algAgn and \algAwareBiLevel in terms of reward achieved by the learner after teaching for object worlds of size $10 \times 10$ in Table~\ref{tab:known-preferences-performance}.
The numbers show the average reward over 10 randomly generated object-worlds. 
Note that \algAwareBiLevel has to solve a non-convex optimization problem to find the optimal teaching policy, cf.\ Eq.~\ref{sec4.eq.bilevel-teacher}.
Because we use a gradient-based optimization approach, the teaching policies found can depend on the initial point for optimization.
Hence, we always consider the following two initial points for optimization and select the teaching policy which results in a higher objective value: (i) all optimization variables in Eq.~\ref{sec4.eq.bilevel-teacher} are set to zero, and
(ii) the optimization variables are initialized as $\alpha^\low[i]= \max\{w_{\bm{\lambda}}[i], 0 \}$, $\alpha^\up[i]= \max\{-w_{\bm{\lambda}}[i], 0 \}$, and $\bm{\beta}=0$, where $\bm{w}_{\bm{\lambda}}$ is as inferred by the learner when taught by \algAgn and $i \in \{ 1, \ldots, d_r\}$, cf.\ Section~\ref{sec:learner:soft}.
From Table~\ref{tab:known-preferences-performance} we observe that a learner can learn better policies from a teacher that accounts for the learner's preferences. 

\begin{table}[!htbp]
    \centering
    \caption{Learners' average rewards after teaching. \textsc{L1}, $\ldots$, \textsc{L5} correspond to learners with preferences as shown in Figure~\ref{fig:exp-known-learner-preferences}. Results are averaged over 10 random object-worlds, $\pm$ standard error}
    \begin{tabular}{r@{\hspace{0.4\tabcolsep}}c@{\hspace{0.4\tabcolsep}}cccccc}
         \toprule
          & & \multicolumn{5}{c}{\bfseries Learner ($C_r=5, C_c=10$)} \\\cmidrule{3-7} 
          & &  \textsc{L1} & \textsc{L2} & \textsc{L3} & \textsc{L4} & \textsc{L5} \\\midrule
         \multirow{2}{*}[-0.3em]{\bfseries Teacher}& \algAgn & $7.99\pm 0.02$ & $0.01 \pm 0.00$ & $0.01  \pm 0.00$ & $0.01  \pm 0.00$ & $0.00 \pm 0.00$ \\[0.5em]
         & \algAwareBiLevel & $8.00 \pm 0.02$ & $7.20 \pm 0.01$ & $4.86 \pm 0.30$ & $3.15 \pm 0.27$ & $1.30 \pm 0.07$ \\
         \bottomrule
    \end{tabular}
    \vspace{-3mm}
    \label{tab:known-preferences-performance}
\end{table}

\subsection{Teaching under unknown constraints}\label{sec:experiments.unknown}
In this section we evaluate the teaching algorithms from  Section~\ref{sec:teacher.unknown}. We consider the learner model from Section~\ref{sec:learner:hard} that uses $L^2$-projection to match reward feature expectations as studied in Section~\ref{sec:teacher.unknown}, cf. Eq.~\ref{sec3.1.eq.generic-learner--limitingcase}.\footnote{To implement the learner in Eq.~\ref{sec3.1.eq.generic-learner--limitingcase}, we approximated the learner's projection onto the set $\Omega_r^{\learner}$ as follows: We implemented the learner based on the optimization problem given in Eq.~\ref{sec3.eq.soft-learner} with a hard constraint on preferences and $L^2$ norm penalty on reward mismatch scaled with a large value of $C_r = 20$. \label{footnote.l2-projection}}
 For modeling the hard constraints, we consider  box-type linear constraints with $\delta^{\hard}_{c}[j]= 2.5 \ \forall j \in \{1, 2, \ldots, d_c\}$ for the preference features, cf. Eq.~\ref{sec3.eq.soft-learner}.

We study the learners L1, L2, and L3 with preferences corresponding to the first three object-worlds shown in Figure~\ref{fig:exp-known-learner-preferences.a}. We report the results for learner L2 below; results for learners L1 and L3 are deferred to the
\iftoggle{longversion}{%
Appendix~\ref{appendix:experiments.unknown}.%
}
{%
Appendix B.2 of the supplementary material.%
}

In this context it is instructive to investigate how quickly these adaptive teaching strategies converge to the performance of a teacher who has full knowledge about the learner. Results comparing the adaptive teaching strategies (\algAdAwareGreedy and \algAdAwareLine) are shown in Figure~\ref{fig:L2-exp-unknown-learner-preferences.plot}. We can observe that both teaching strategies get close to the best possible performance under full knowledge about the learner (\algAwareCMDP). We also provide results showing the performance achieved by the adaptive teaching strategies on object-worlds of varying sizes, see Figure~\ref{fig:L2-exp-unknown-learner-preferences.table}.  

Note that the performance of \algAdAwareGreedy{} decreases slightly when teaching for more rounds, i.e., comparing the results after 3 teaching rounds and at the end of the teaching process. This is because of approximations when learner is computing the policy via projection, which in turn leads to errors on the teacher side when approximating $\hat \Omega_r^{\learner}$ (refer to discussion in Footnote~\ref{footnote.l2-projection}). In contrast, \algAdAwareLine{} performance always increases when teaching for more rounds.

\begin{figure*}[!htb]
\centering
\begin{minipage}[b]{0.3\textwidth}
\begin{subfigure}[b]{1\textwidth}
\centering
    \includegraphics[width=5.5cm]{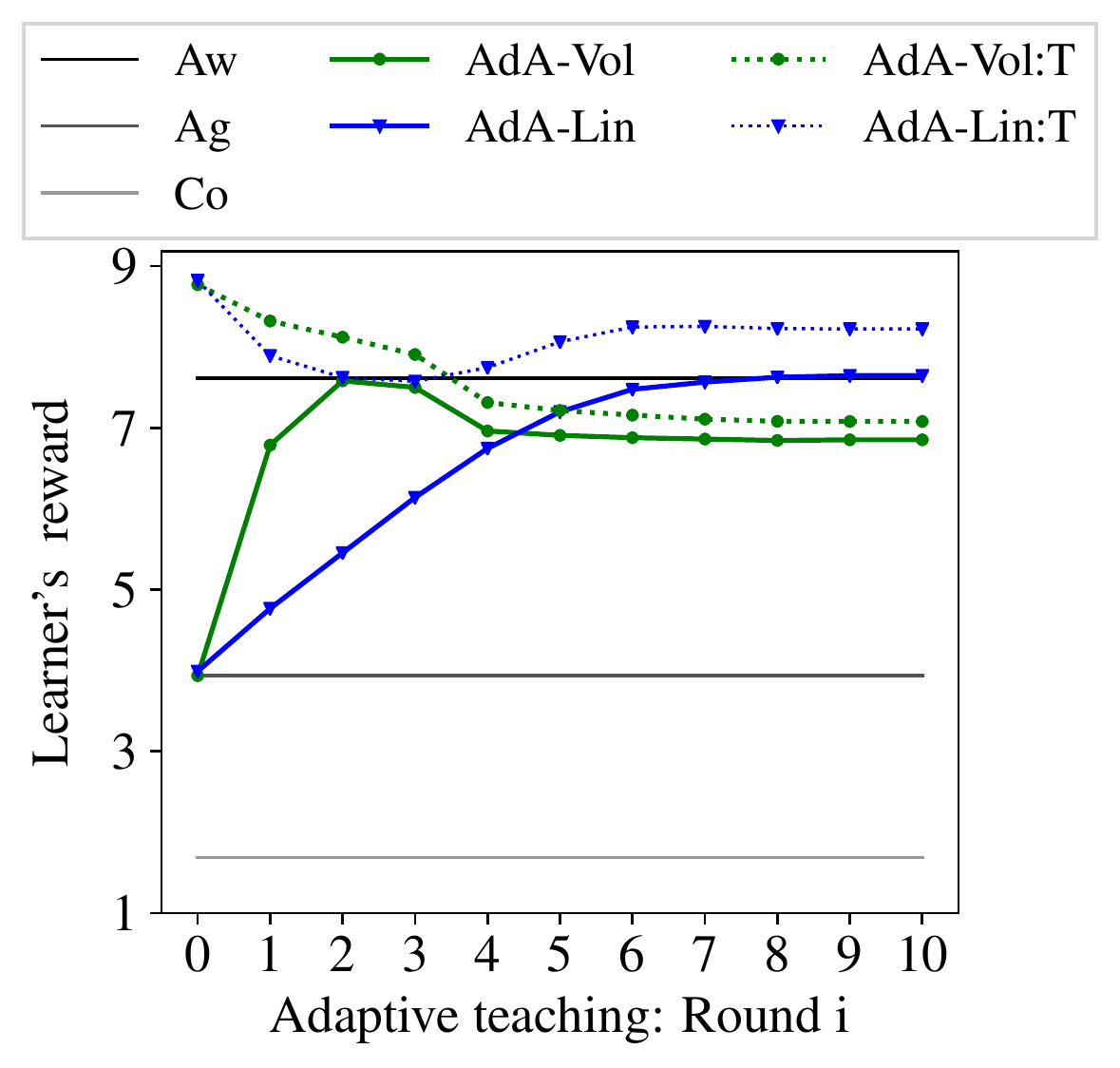}
    \caption{Reward over teaching rounds}
    \label{fig:L2-exp-unknown-learner-preferences.plot}
\end{subfigure}
\end{minipage}
\qquad \quad 
\begin{minipage}[b]{0.6\textwidth}
\begin{subfigure}[b]{1\textwidth}
\centering
\begingroup
\renewcommand{\arraystretch}{1.25} 
	\centering
    \centering
    \scalebox{0.815}{
    \begin{tabular}{c|cccc}
         \toprule
          \backslashbox{Teacher}{Env} &  $10\times10$ & $15\times15$ & $20\times20$ \\\midrule
          \algAwareCMDP & $7.62\pm 0.02$ & $7.44\pm 0.04$ & $7.19\pm 0.04$& \\
          \algAgn & $3.94\pm 0.09$& $3.84\pm 0.06$ & $3.95\pm 0.06$& \\
          \algCon & $1.68\pm 0.01$& $1.67\pm 0.012$ & $	1.62\pm 0.02$& \\
          \midrule
          \algAdAwareGreedy ($3^\textnormal{rd}$) & $7.50\pm 0.14$& $7.50\pm 0.04$ & $7.29\pm 0.05$ & \\
          \algAdAwareGreedy (end) & $6.85\pm 0.33$& $7.06\pm 0.06$ & $6.77\pm 0.08$& \\
          \midrule
          \algAdAwareLine ($3^\textnormal{rd}$) & $6.14\pm 	0.08$& $6.28\pm 0.10$ & $6.37\pm 0.08$& \\
          \algAdAwareLine (end) & $7.64\pm 0.02$& $7.53\pm 0.03$ & $7.29\pm 0.06$& \\
         \bottomrule
    \end{tabular}
    }
    \caption{Varying grid-size}
    \label{fig:L2-exp-unknown-learner-preferences.table}	
\endgroup
\end{subfigure}
\end{minipage}
\caption{Performance of adaptive teaching strategies \algAdAwareGreedy and \algAdAwareLine. \textbf{(left)} Figure~\ref{fig:L2-exp-unknown-learner-preferences.plot} shows the reward for learner's policy over number of teaching interactions. The horizontal lines indicate the performance of learner's policy for the learner-aware teacher with full knowledge of the learner's constraints \algAwareCMDP, the learner-agnostic teacher \algAgn who ignores any constraints, and a conservative teacher \algCon who considers all $6$ constraints (assuming the learner model L5 in Figure~\ref{fig:exp-known-learner-preferences}). Our adaptive teaching strategies \algAdAwareGreedy and \algAdAwareLine significantly outperform baselines (\algAgn and \algCon) and quickly converge towards the optimal performance of \algAwareCMDP. The dotted lines \algAdAwareGreedy:T and \algAdAwareLine:T show the rewards corresponding to teacher's policy at a round and are shown to highlight the very different behavior of two adaptive teaching strategies. \textbf{(right)} Table~\ref{fig:L2-exp-unknown-learner-preferences.table} shows results for varying grid-size of the environment. Results are reported at $i=3^\textnormal{rd}$ round and at the ``end" round when algorithm reaches it's stopping criterion. Results are reported as average over 10 runs $\pm$ standard error, where each run corresponds to a random environment.
\vspace{-2mm}}
\label{fig:L2-exp-unknown-learner-preferences}
\vspace{-2mm}
\end{figure*}

\section{Related Work}\label{sec:relatedwork}
Our work is closely related to algorithmic machine
teaching
\cite{goldman1995complexity,zhu2015machine,zhu_overview_2018}, whose
general goal is to design teaching algorithms that optimize the data
that is provided to a learning algorithm. Most works in machine
teaching so far focus on supervised learning tasks and assume that the learning algorithm is fully known to the teacher,
see e.g.,\ \cite{zhu2013machine,singla2014near,liu2016teaching,mac_aodha_teaching_2018}.

In the IRL setting, few works study how to provide maximally informative demonstrations to the learner, e.g., \cite{cakmak2012algorithmic,danielbrown2018irl}. 
In contrast to our work, their teacher fully knows the learner model and 
provides the demonstrations without any adaptation
to the learner. The question of how a teacher should
adaptively react to a learner has been addressed by
\cite{singla2013actively,liu2017blackbox,chen2018understanding,melo2018interactive,DBLP:conf/aaai/YeoKSMAFDC19,hunziker2018teaching}, but only in the supervised setting.  In a recent work, \cite{DBLP:conf/ijcai/KamalarubanDCS19} have studied the problem of adaptively teaching an IRL agent by providing an informative sequence of demonstrations. However, they assume that the teacher has full knowlege of the learner's dynamics.

%
%
Within the area of IRL, there is a line of work on active learning
approaches
\cite{cohn_comparing_2011,brown_risk-aware_2018,brown2018efficient,kareem2018_repeated,cui_active_2018},
which is related to our work.
In contrast to us, they take the perspective of the learner who
actively influences the demonstrations it receives. A few papers have
addressed  the problem that arises 
when the learner does not have full access to the
reward features, e.g.,
\cite{levine2010feature} and \cite{haug_teaching_2018}.



Our work is also loosely related to multi-agent reinforcement learning.
\cite{dimitrakakis2017multi} studied the interaction between agents with
misaligned models with a focus on the question of how to jointly optimize a
policy. \cite{ghosh19towardsrobust} studied the problem of designing robust AI agent that can interact with another agent of unknown type. However, these works do not tackle the problem of teaching an agent by demonstrations. Another related work is \cite{hadfield2016cooperative} which studied the cooperation of agents who do not perfectly understand each other.




\vspace{-3mm}
\section{Conclusions and Outlook}\label{sec:conclusions}
\vspace{-2mm}


In this paper we considered inverse reinforcement learning in the context of learners with preferences and constraints.
In this setting, the learner does not only focus on matching the teacher's demonstrated behavior but also takes its own preferences, e.g., behavioral biases or physical constraints, into account.
We developed a theoretical framework for this setting, and proposed and studied algorithms for learner-aware teaching in which the teacher accounts for the learner's preferences for the cases of known and unknown preference constraints.
We demonstrated significant performance improvements of our learner-aware teaching strategies as compared to learner-agnostic teaching both theoretically and empirically.
Our theoretical framework and our proposed algorithms foster the application of IRL in real-world settings in which the learner does not blindly follow a teacher's demonstrations.

There are several promising directions for future work, including but not limited to: The evaluation of our approach in machine-human and human-machine tasks; extensions of our approach to other learner models; approaches for learning efficiently from a learner's point of view from a fixed set of (potentially suboptimal) demonstrations in the case of preference constraints.
 






\clearpage
\subsubsection*{Acknowledgements}
This work was supported by Microsoft Research through its PhD Scholarship Programme.

\bibliography{references}

\iftoggle{longversion}{
\clearpage
\onecolumn
\appendix 
{\allowdisplaybreaks
\section{List of Appendices}\label{appendix:table-of-contents}
In this section we provide a brief description of the content provided in the appendices of the paper.   
\begin{itemize}
\item Appendix~\ref{appendix:experiments} provides additional experimental results (Section~\ref{sec:experiments}).
\item Appendix~\ref{appendix:teacher.unknown} provides additional details on the adaptive teaching strategies (Section~\ref{sec:teacher.unknown}).
\item Appendix~\ref{appendix:background-mceirl} provides background on the (discounted)   MCE-IRL problem (Section~\ref{sec:learner}).
\item Appendix~\ref{appendix:mceirl-with-preferences} provides additional details on the (discounted) MCE-IRL problem with preferences  (Section~\ref{sec:learner:soft}).
\item Appendix~\ref{appendix:lp} provides the LP formulation for the teacher \algAwareCMDP (Section~\ref{sec:teacher.fullknowledge.special}).
\item Appendix~\ref{appendix:bi-level} provides additional details on the bi-level optimization approach for the teacher \algAwareBiLevel (Section \ref{sec:teacher.fullknowledge.generic}).
\end{itemize}

\clearpage
\section{Experimental Evaluation: Additional Results (Section~\ref{sec:experiments})}
\label{appendix:experiments}

\subsection{Teaching under known constraints (Section~\ref{sec:experiments.known})}\label{appendix:experiments.known}

Additional results for teaching under known constraints are presented in Table~\ref{tab:known-preferences-performance-additional-results}.
We observe that \algAwareBiLevel clearly outperforms \algAgn for most combinations of $C_r$ and $C_c$.
Only for $C_r=10, C_c = 1$, the teachers \algAwareBiLevel and \algAgn achieve similar performance because $C_r \gg C_c$, and hence the learner values achieving higher reward more than satisfying its preferences.

\begin{table}[!htbp]
    \centering
    \caption{Learners' average rewards after teaching. \textsc{L1}, $\ldots$, \textsc{L5} correspond to learners with preferences as shown in Figure~\ref{fig:exp-known-learner-preferences}. Results are averaged over 10 random object-worlds, $\pm$ standard error}
    \begin{tabular}{r@{\hspace{0.4\tabcolsep}}c@{\hspace{0.4\tabcolsep}}cccccc}
         \toprule
          & & \multicolumn{5}{c}{\bfseries Learner ($C_r=5, C_c=10$)} \\\cmidrule{3-7} 
          & &  \textsc{L1} & \textsc{L2} & \textsc{L3} & \textsc{L4} & \textsc{L5} \\\midrule
         \multirow{2}{*}[-0.3em]{\bfseries Teacher}& \algAgn & $7.99\pm 0.02$ & $0.01 \pm 0.00$ & $0.01  \pm 0.00$ & $0.01  \pm 0.00$ & $0.00 \pm 0.00$ \\[0.5em]
         & \algAwareBiLevel & $8.00 \pm 0.02$ & $7.20 \pm 0.01$ & $4.86 \pm 0.30$ & $3.15 \pm 0.27$ & $1.30 \pm 0.07$ \\
         \bottomrule
    \end{tabular}
    \\[3mm]
   \begin{tabular}{r@{\hspace{0.4\tabcolsep}}c@{\hspace{0.4\tabcolsep}}cccccc}
         \toprule
          & & \multicolumn{5}{c}{\bfseries Learner ($C_r=10, C_c=10$)} \\\cmidrule{3-7} 
          & &  \textsc{L1} & \textsc{L2} & \textsc{L3} & \textsc{L4} & \textsc{L5} \\\midrule
         \multirow{2}{*}[-0.3em]{\bfseries Teacher}& \algAgn & $8.34\pm 0.01$ & $0.17 \pm 0.02$ & $0.01  \pm 0.00$ & $0.01  \pm 0.00$ & $0.00 \pm 0.00$ \\[0.5em]
         & \algAwareBiLevel & $8.33 \pm 0.01$ & $6.90 \pm 0.17$ & $5.03 \pm 0.31$ & $3.27 \pm 0.28$ & $1.35 \pm 0.07$ \\
         \bottomrule
    \end{tabular}
    \\[3mm]
   \begin{tabular}{r@{\hspace{0.4\tabcolsep}}c@{\hspace{0.4\tabcolsep}}cccccc}
         \toprule
          & & \multicolumn{5}{c}{\bfseries Learner ($C_r=10, C_c=5$)} \\\cmidrule{3-7} 
          & &  \textsc{L1} & \textsc{L2} & \textsc{L3} & \textsc{L4} & \textsc{L5} \\\midrule
         \multirow{2}{*}[-0.3em]{\bfseries Teacher}& \algAgn & $8.36\pm 0.01$ & $8.14 \pm 0.03$ & $0.01  \pm 0.00$ & $0.01  \pm 0.00$ & $0.00 \pm 0.00$ \\[0.5em]
         & \algAwareBiLevel & $8.34 \pm 0.01$ & $8.13 \pm 0.03$ & $5.20 \pm 0.29$ & $3.43 \pm 0.27$ & $1.69 \pm 0.0$ \\
         \bottomrule
    \end{tabular}
    \\[3mm]
   \begin{tabular}{r@{\hspace{0.4\tabcolsep}}c@{\hspace{0.4\tabcolsep}}cccccc}
         \toprule
          & & \multicolumn{5}{c}{Learner ($C_r=5, C_c=5$)} \\\cmidrule{3-7} 
          & &  \textsc{L1} & \textsc{L2} & \textsc{L3} & \textsc{L4} & \textsc{L5} \\\midrule
         \multirow{2}{*}[-0.3em]{\bfseries Teacher}& \algAgn & $7.99\pm 0.02$ & $0.17 \pm 0.02$ & $0.01  \pm 0.00$ & $0.01  \pm 0.00$ & $0.00 \pm 0.00$ \\[0.5em]
         & \algAwareBiLevel & $8.00 \pm 0.02$ & $6.64 \pm 0.17$ & $4.87 \pm 0.30$ & $3.16 \pm 0.27$ & $1.31 \pm 0.06$ \\
         \bottomrule
    \end{tabular}
    \\[3mm]
   \begin{tabular}{r@{\hspace{0.4\tabcolsep}}c@{\hspace{0.4\tabcolsep}}cccccc}
         \toprule
          & & \multicolumn{5}{c}{\bfseries Learner ($C_r=10, C_c=1$)} \\\cmidrule{3-7} 
          & &  \textsc{L1} & \textsc{L2} & \textsc{L3} & \textsc{L4} & \textsc{L5} \\\midrule
         \multirow{2}{*}[-0.3em]{\bfseries Teacher}& \algAgn & $8.36\pm 0.01$ & $8.39 \pm 0.02$ & $8.46  \pm 0.02$ & $8.46  \pm 0.02$ & $8.49 \pm 0.02$ \\[0.5em]
         & \algAwareBiLevel & $8.33 \pm 0.01$ & $8.36 \pm 0.03$ & $8.44 \pm 0.02$ & $8.44 \pm 0.02$ & $8.46 \pm 0.02$ \\
         \bottomrule
    \end{tabular}
    \\[3mm]
   \begin{tabular}{r@{\hspace{0.4\tabcolsep}}c@{\hspace{0.4\tabcolsep}}cccccc}
         \toprule
          & & \multicolumn{5}{c}{\bfseries Learner ($C_r=1, C_c=10$)} \\\cmidrule{3-7} 
          & &  \textsc{L1} & \textsc{L2} & \textsc{L3} & \textsc{L4} & \textsc{L5} \\\midrule
         \multirow{2}{*}[-0.3em]{\bfseries Teacher}& \algAgn & $5.67\pm 0.02$ & $0.15 \pm 0.02$ & $0.16  \pm 0.02$ & $0.11  \pm 0.01$ & $0.08 \pm 0.01$ \\[0.5em]
         & \algAwareBiLevel & $5.93 \pm 0.02$ & $4.49 \pm 0.15$ & $3.56 \pm 0.24$ & $2.30 \pm 0.22$ & $0.93 \pm 0.05$ \\
         \bottomrule
    \end{tabular}
    
    \label{tab:known-preferences-performance-additional-results}
\end{table}

\clearpage
\subsection{Teaching under unknown constraints (Section~\ref{sec:experiments.unknown})}\label{appendix:experiments.unknown}
Here, we provide additional experimental results for teaching algorithms from  Section~\ref{sec:teacher.unknown}. In particular, we report on the results for learner L1 and learner L3, similar to the results for learner L2 reported in Section~\ref{sec:experiments.unknown}.


\begin{figure*}[!htb]
\centering
\begin{minipage}[b]{0.3\textwidth}
\begin{subfigure}[b]{1\textwidth}
\centering
    \includegraphics[width=5.5cm]{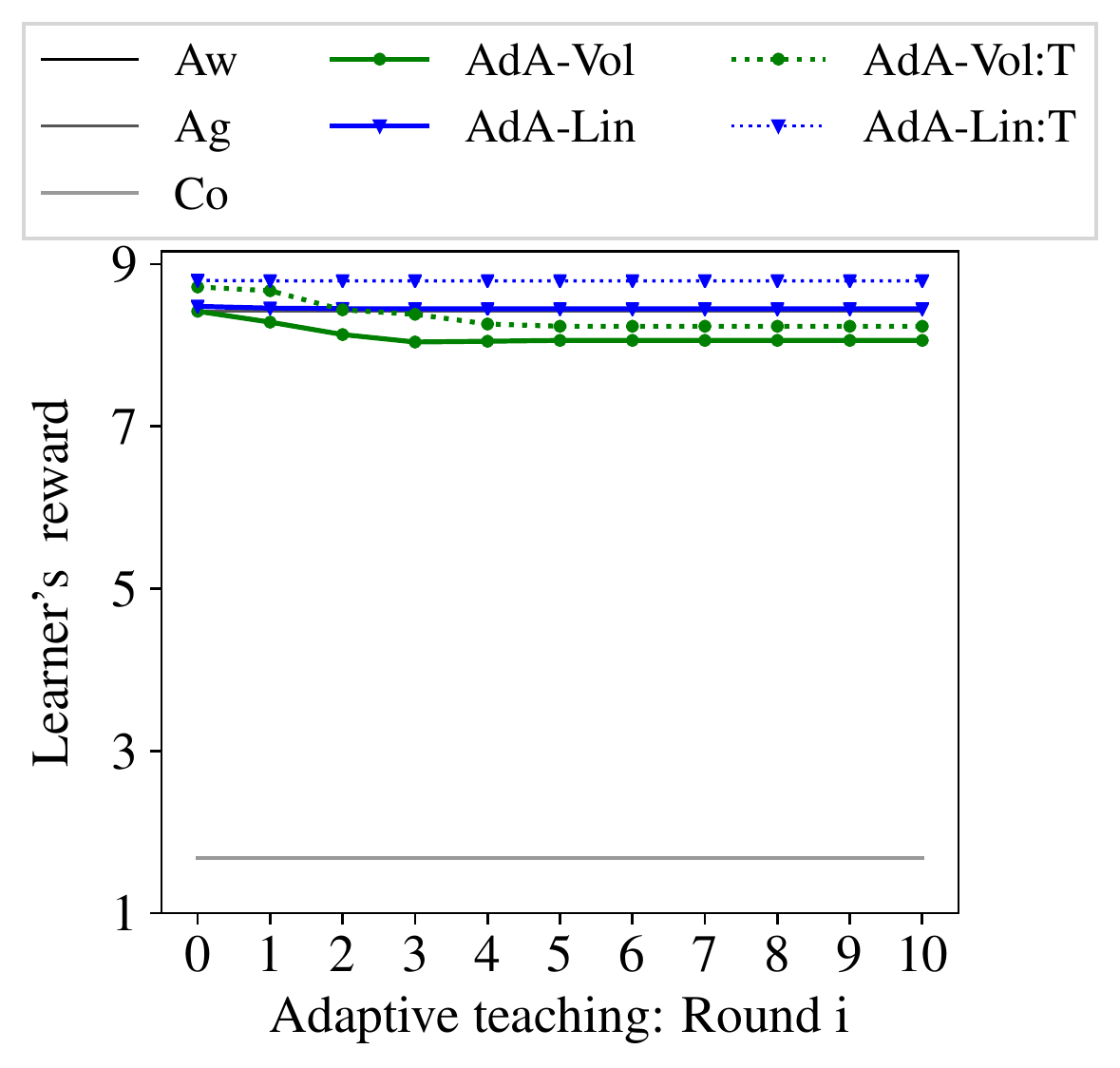}
    \caption{Reward over teaching rounds}
    \label{fig:L1-exp-unknown-learner-preferences.plot}
\end{subfigure}
\end{minipage}
\qquad \quad 
\begin{minipage}[b]{0.6\textwidth}
\begin{subfigure}[b]{1\textwidth}
\centering
\begingroup
\renewcommand{\arraystretch}{1.25} 
	\centering
    \centering
    \scalebox{0.815}{
    \begin{tabular}{c|cccc}
         \toprule
          \backslashbox{Teacher}{Env} &  $10\times10$ & $15\times15$ & $20\times20$ \\\midrule
          \algAwareCMDP & $8.42\pm 0.03$ & $8.24\pm 0.05$ & $7.84\pm 0.08$& \\
          \algAgn & $8.42\pm 0.03$& $8.24\pm 0.05$ & $7.84\pm 0.08$& \\
          \algCon & $1.68\pm 0.1$& $1.66\pm 0.01$ & $1.65\pm 0.02$& \\
          \midrule
          \algAdAwareGreedy ($3^\textnormal{rd}$) & $8.04\pm 0.02$& $7.83\pm 0.04$ & $7.46\pm 0.07$ & \\
          \algAdAwareGreedy (end) & $8.06\pm 0.02$& $7.80\pm 0.08$ & $7.30\pm 0.12$& \\
          \midrule
          \algAdAwareLine ($3^\textnormal{rd}$) & $8.44\pm 0.04$& $8.23\pm 0.07$ & $8.08\pm 0.08$& \\
          \algAdAwareLine (end) & $8.44\pm 0.04$& $8.23\pm 0.07$ & $8.08\pm 0.08$& \\
         \bottomrule
    \end{tabular}
    }
    \caption{Varying grid-size}
\endgroup
\end{subfigure}
\end{minipage}
\label{fig:L1-exp-unknown-learner-preferences}
\caption{Results for learner L1}
\end{figure*}


\begin{figure*}[!htb]
\centering
\begin{minipage}[b]{0.3\textwidth}
\begin{subfigure}[b]{1\textwidth}
\centering
    \includegraphics[width=5.5cm]{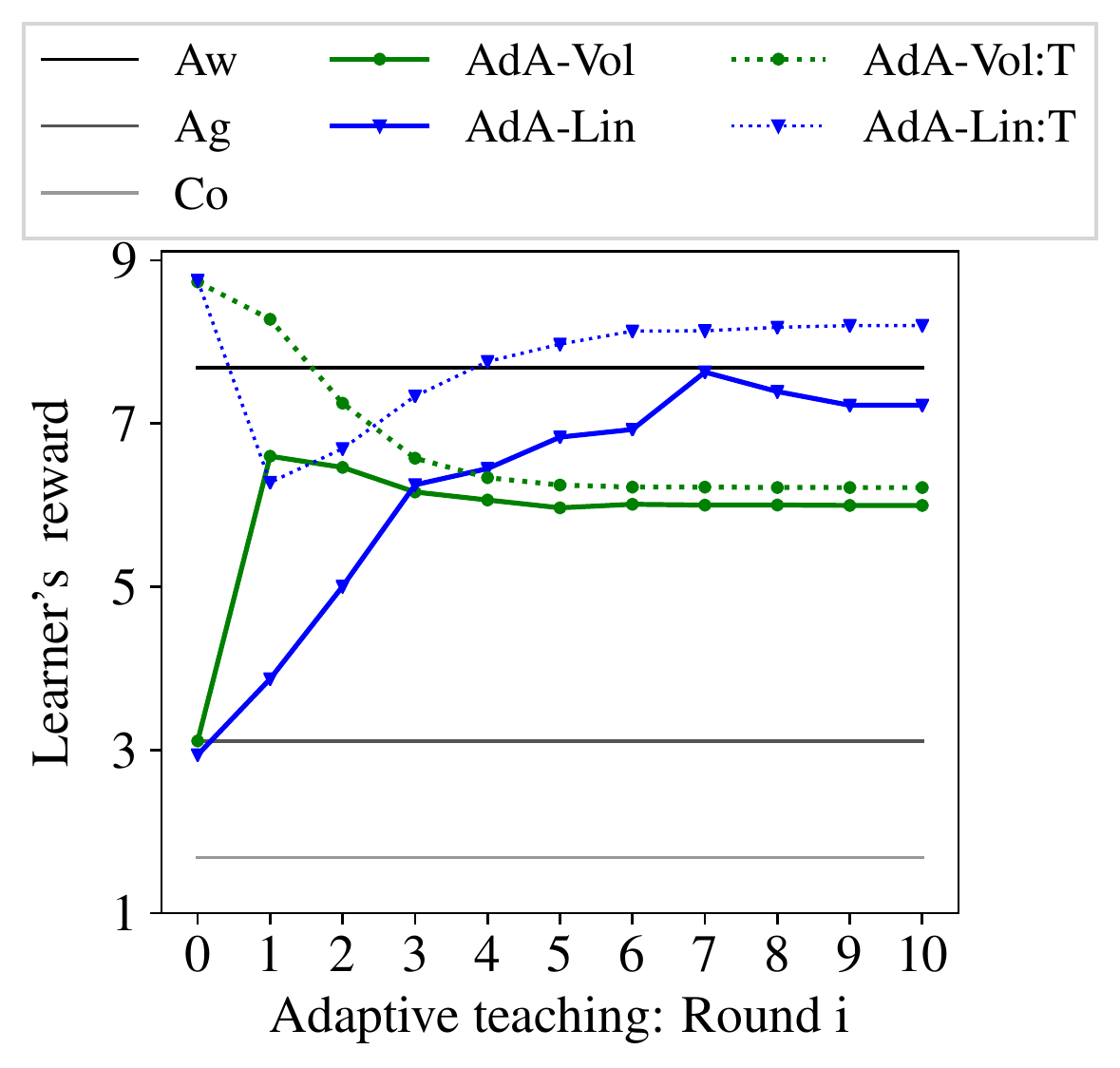}
    \caption{Reward over teaching rounds}
    \label{fig:L3-exp-unknown-learner-preferences.plot}
\end{subfigure}
\end{minipage}
\qquad \quad 
\begin{minipage}[b]{0.6\textwidth}
\begin{subfigure}[b]{1\textwidth}
\centering
\begingroup
\renewcommand{\arraystretch}{1.25} 
	\centering
    \centering
    \scalebox{0.815}{    
    \begin{tabular}{c|cccc}
         \toprule
          \backslashbox{Teacher}{Env} &  $10\times10$ & $15\times15$ & $20\times20$ \\\midrule
          \algAwareCMDP & $7.68\pm 0.04$ & $7.35\pm 0.03$ & $7.39\pm 0.09$& \\
          \algAgn & $3.11\pm 0.08$& $3.12\pm 0.07$ & $3.26\pm 0.14$& \\
          \algCon & $1.68\pm 0.01$& $1.65\pm 0.01$ & $1.62\pm 0.01$& \\
          \midrule
          \algAdAwareGreedy ($3^\textnormal{rd}$) & $6.16\pm 0.42$& $5.72\pm 0.54$ & $6.39\pm 0.32$ & \\
          \algAdAwareGreedy (end) & $5.99\pm 0.46$& $5.38\pm 0.56$ & $6.16\pm 0.31$& \\
          \midrule
          \algAdAwareLine ($3^\textnormal{rd}$) & $6.25\pm 0.20$& $5.13\pm 0.50$ & $6.15\pm 	0.11$& \\
          \algAdAwareLine (end) & $7.22\pm 0.16$& $5.83\pm 0.62$ & $7.09\pm 0.07$& \\
         \bottomrule
    \end{tabular}
    }
    \caption{Varying grid-size}
    \label{fig:L3-exp-unknown-learner-preferences.table}	
\endgroup
\end{subfigure}
\end{minipage}
\caption{Results for learner L3}
\label{fig:L3-exp-unknown-learner-preferences}
\end{figure*}

\clearpage
\section{Details for Learner-Aware Teaching under Unknown Constraints
(Section~\ref{sec:teacher.unknown})}\label{appendix:teacher.unknown}

In this appendix, we provide more details on the adaptive teaching
algorithms \algAdAwareGreedy and \algAdAwareLine described in Sections
\ref{sec:teacher.unknown.adaptive-greedy} and
\ref{sec:teacher.unknown.adaptive-blackbox}. Recall that both teaching
algorithms are obtained from Algorithm \ref{algo:interaction1} by
defining the way in which the teacher $\teacher$ adapts the teaching policy based
on the learner $\learner$'s feature expectations $\mu_r^{\learner}$ in past
rounds.

\subsection{Details for \algAdAwareGreedy (Section \ref{sec:teacher.unknown.adaptive-greedy})}
\label{sec:appendix.teacher.unknown.greedy}
\paragraph{Estimation of the learner's constraint set.} In
\algAdAwareGreedy, $\teacher$ maintains an estimate
$\hat \Omega_r^{\learner, i}$ of $\learner$'s constraint set, starting
with $\hat \Omega_r^{\learner, 0} = \Omega_r$. After observing the
feature expectations $\mu_r^{\learner, i}$ of the policy $\learner$
found in round $i$, $\teacher$ updates this estimate as follows:
\begin{equation}
    \label{eq:update-Omega}
    \hat \Omega_r^{\learner, i + 1} ~:=~  \hat
    \Omega_r^{\learner, i} \cap
    \{\mu_r^{\learner, i} + \nu \in \R^{d_r} ~|~ \langle \mu_r^{\teacher,
        i} - \mu_r^{\learner, i}, \nu\rangle \leq 0 \}
\end{equation}
The set on the right hand side of \eqref{eq:update-Omega} with which
$\Omega_r^{\learner, i}$ gets intersected is a halfspace 
containing $\Omega_r^{\learner}$. This is due to the fact that
$\Omega_r^\learner$ is convex by assumption, and to our assumption that
$\learner$'s learning algorithm is such that it outputs a policy whose
feature expectations $\mu_r^{\learner, i}$ match the $L^2$-projection
of $\mu_r^{\teacher, i}$ to $\Omega_r^{\learner}$. Inductively, it
follows that $\hat \Omega_r^{\learner, i} \supset \Omega_r^{\learner}$
for all $i$.

In practice, we implement a slightly modified version of the update
step in which we intersect $\hat \Omega_r^{\learner, i}$ with a
halfspace that is shifted in the direction of
$\mu_r^{\teacher, i} - \mu_r^{\learner, i}$ by a small amount, i.e., we use
\begin{equation*}
    \{\mu_r^{\learner, i} + (1- \eta) (\mu_r^{\teacher, i} -
    \mu_r^{\learner, i}) + \nu \in \R^{d_r} ~|~ \langle \mu_r^{\teacher,
        i} - \mu_r^{\learner, i}, \nu\rangle \leq 0 \}
\end{equation*}
with a step size parameter $\eta \in (0,1)$. This helps make the algorithm
more robust to noise in the learner's feature expectations. In our
experiments, we used $\eta = 0.9$.

\paragraph{Update of the teaching policy.} After updating the estimate
of the learner's constraint set to $\hat \Omega_r^{\learner, i}$,
$\teacher$ solves a constrained MDP in order to find
\begin{equation*}
    \pi^{\teacher, i + 1} \in \argmax_{\pi, \mu_r(\pi) \in \hat
        \Omega_r^{\learner, i}} R(\pi).
\end{equation*}

Given that $\hat \Omega_r^{\learner, i}$ is cut out by linear
equations, solving the constrained MDP reduces to solving an LP, as
described in Appendix \ref{appendix:lp}.

\paragraph{Termination of the interaction.} The algorithm terminates
as soon as the stopping criterion
$\Vert \mu_r^{\learner, i} - \mu_r^{\teacher, i} \Vert_2 \leq
\epsilon$ is satisfied. Note that
$\hat \Omega_r^{\learner, i} \supset \Omega_r^{\learner}$ implies that
\begin{equation*}
    R(\pi^{\teacher, i}) \geq R(\pi^\textnormal{aware})
\end{equation*}
for any
$\pi^\textnormal{aware} \in \argmax_{\pi, \mu_r(\pi) \in \Omega_r^{\learner}}
R(\pi)$. Therefore, after termination we have
\begin{equation*}
    R(\pi^{\learner, i}) \geq R(\pi^\textnormal{aware}) - \epsilon
\end{equation*}
for any policy $\pi^\textnormal{aware}$ which is optimal under $\learner$'s
constraints, which is the first statement of Theorem~\ref{thm:teacher.unknown.adaptive-greedy}.

The second statement of Theorem
\ref{thm:teacher.unknown.adaptive-greedy} follows from the fact that
if $\Omega_r^{\learner}$ is a convex polytope cut out by $m$ linear
inequalities, the number of faces, which is in $O(m^{d_r})$, is 
an upper bound on the number of iterations of the
algorithm, because one face is ``eliminated'' in each round.

\subsection{Details for \algAdAwareLine (Section \ref{sec:teacher.unknown.adaptive-blackbox})}
\label{appendix:teacher.unknown.linesearch}

\begin{figure}[t]\centering
    \begin{minipage}[b]{0.5\linewidth}
        \begin{algorithm}[H]
            \caption{\algLineSearch}
            \label{algo:line-search}
            \begin{algorithmic}[1]
                \Require $\mu_r^{\learner}$, $\alpha_{\min}$, $\alpha_{\max}$,
                $\varepsilon_\alpha$, $\varepsilon_{\mu}$.
                \State $\alpha_u \leftarrow \alpha_{\max}$, $\alpha_l \leftarrow \alpha_{\min}$
                \While{$\alpha_u - \alpha_l > \varepsilon_\alpha$}
                \State $\alpha \leftarrow (\alpha_u + \alpha_l)/2$
                \State
                $\pi^{\teacher} \leftarrow
                \textsc{IRL}(\mu_r^{\learner} + \alpha \wopt)$
                \If{$\Vert \mu_r(\pi^{\teacher}) - \mu_r^{\learner}
                    - \alpha \wopt \Vert_2 >
                    \varepsilon_\mu$}
                \State $\alpha_u \leftarrow \alpha$
                \Else
                \State $\alpha_l \leftarrow \alpha$
                \EndIf{}
                \EndWhile{}
                \If{$\Vert \mu_r(\pi^{\teacher}) - \mu_r^{\learner} - \alpha \wopt \Vert_2 >
                    \varepsilon_\mu$}
                \State
                $\pi^{\teacher} \leftarrow
                \textsc{IRL}(\mu_r^{\learner} + \alpha_{\min}
                \wopt)$
                \EndIf{}
                \State \Return $\pi^\teacher$
            \end{algorithmic}
        \end{algorithm}
    \end{minipage}
    \hfill
    \begin{minipage}[b]{0.48\linewidth}
        \begin{figure}[H]
            \centering
            \includegraphics{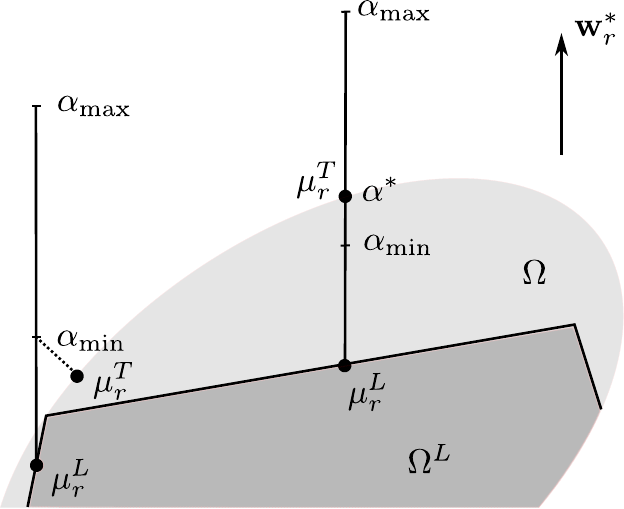}
            \caption{
            }
            \label{fig:line-search}
        \end{figure}
    \end{minipage}
    \caption*{\algLineSearch is the algorithm that $\teacher$ uses in
        order to find a teaching policy $\pi^{\teacher}$ provided that
        the feature expectations of $\learner$'s current policy are
        $\mu_r^\learner$.  Figure \ref{fig:line-search} illustrates
        the two cases may occur: For the right $\mu_r^{\learner}$,
        \algLineSearch returns a policy $\pi^{\teacher}$ whose feature
        expectations satisfy
        $\mu_r^{\teacher} = \mu_r^{\learner} + \alpha^* \wopt$ such
        that $\alpha^* > \alpha_{\min}$. For the left
        $\mu_r^{\learner}$, \algLineSearch returns a policy
        $\pi^{\teacher}$ whose feature expectations satisfy
        $\mu_r^{\teacher} \in \argmin_{\mu_r \in \Omega_r} \Vert \mu_r
        - \mu_r^{\learner} + \alpha_{\min} \mu_r^{\teacher}\Vert$.}
\end{figure}

In \algAdAwareLine, $\teacher$ updates the teaching policy
$\pi^{\teacher, i+1}$ based on $\learner$'s feature expectations
$\mu_r^{\learner, i}$ from the previous round. To do so, $\teacher$
uses \algLineSearch (Algorithm \ref{algo:line-search}) to perform a
binary search on the line segment
\begin{equation}
    \label{eq:line-segment}
    \{\mu_r^{\learner, i} + \alpha \wopt ~|~ \alpha \in [\alpha_{\min},
    \alpha_{\max}]\} \subset \R^{d_r}
\end{equation}
in order to find a vector $\mu_r$ that is realizable as the vector of
feature expectations of a policy. 
If the intersection of the line segment \eqref{eq:line-segment} with
$\Omega_r$ is non-empty, it is of the form
$\{\mu_r^{\learner} + \alpha \wopt ~|~ \alpha \in [\alpha_{\min},
\alpha^*]\}$ for some $\alpha^* \leq \alpha_{\max}$ due to the
convexity of $\Omega_r$. In that case, \algLineSearch returns a policy
with feature expectations
\begin{equation*}
    \mu_r^{\teacher, i+1} = \mu_r^{\learner, i} + \alpha_i^*
    \wopt,
\end{equation*}
where $\alpha_i^*$ is the maximal
$\alpha \in [\alpha_{\min}, \alpha_{\max}]$ such that
$\mu_r^{\learner, i} + \alpha \wopt \in \Omega_r$. If the intersection
is empty, \algLineSearch returns a policy with feature expectations
\begin{equation*}
    \mu_r^{\teacher, i + 1} \in \argmin_{\mu_r \in \Omega_r} \Vert \mu_r
    - \mu_r^{\learner, i} - \alpha_{\min} \wopt \Vert_2.
\end{equation*}
Figure \ref{fig:line-search} illustrates the two cases that
may occur.




\subsubsection{Proof of Theorem \ref{thm:convergence}}
In this section, we provide the proof of Theorem \ref{thm:convergence}, which
gives a guarantee on the improvement of $\learner$'s performance in each round of
the \algAdAwareLine algorithm. The assumption we make here is that, in every
teaching round, \algLineSearch returns a teaching policy
$\pi^{\teacher, i+1}$ such that
$\mu_r^{\teacher, i+1} = \mu_r^{\learner, i} + \alpha_i \wopt$ for
some $\alpha_i \geq \alpha_{\min}$, where $\alpha_{\min} > 0$ is a
fixed constant. It is easy to see that this assumption, together with our assumption
on $\learner$'s algorithm and the convexity of $\Omega_r^{\learner}$,
imply that the change in learner performance
\begin{equation*}
    \Delta R_i := R(\mu_r^{\learner, i+1}) - R(\mu_r^{\learner,
        i})
\end{equation*}
is non-negative in every teaching round. The following proposition,
which will be needed in the proof of Theorem \ref{thm:convergence},
strengthens this statement:


\begin{proposition}
    \label{prop:perf-increase}
    Let $\overline R_\learner := \max_{\mu_r \in \Omega_r^{\learner}} R(\mu_r)$ be
    the maximally achievable learner performance. Assume that, in
    teaching round $i$, $\teacher$ can find a teaching policy
    $\pi^{\teacher, i+1}$ whose feature expectations satisfy
    $\mu_r^{\teacher, i+1} = \mu_r^{\learner, i} + \alpha_i \wopt$ for
    some $\alpha_i > 0$. Then
    \begin{equation}
        \label{eq:gap-estimate}
        \overline R_\learner - R(\mu_r^{\learner, i}) \leq 
        \Delta R_i + D \cdot \sqrt{\frac{\Delta R_i}{\alpha_i - \Delta
                R_i}},
    \end{equation}
    where $D = \diam \Omega_r$.
\end{proposition}

\begin{proof}[Proof of Proposition \ref{prop:perf-increase}]
    Consider the plane $V \subset \R^{d_r}$ spanned by
    $\mu_r^{\learner, i}, \mu_r^{\teacher, i+1}$ and
    $\mu_r^{\learner, i+1}$ and denote by $\tilde \mu_r$ the unique
    point in $V$ with the properties that
    \begin{enumerate}[(a)]
    \item $\langle \wopt, \tilde \mu_r \rangle = \langle \wopt,
        \mu_r^{\learner, i+1} \rangle$,
    \item $\tilde \mu_r$ lies on the
        same side of the line through $\mu^{\learner, i}$ and
        $\mu^{\teacher, i+1}$ as $\mu_r^{\learner, i+1}$, and 
    \item $\tilde \mu_r, \mu_r^{\teacher, i+1}$ and $\mu_r^{\learner, i}$
        span a right triangle with $\tilde \mu_r$ at the right-angled corner.
    \end{enumerate}
    Note that $\mu_r^{\learner, i+1}$ must lie inside this triangle,
    i.e., on the red line segment in Figure \ref{fig:increase2}:
    Otherwise there would a point on the line segment connecting
    $\mu_r^{\learner, i+1}$ and $\mu_r^{\learner, i}$, and hence in
    $\Omega_r^\learner$ by convexity, which is closer to
    $\mu_r^{\teacher, i+1}$ than $\mu_r^{\learner, i+1}$,
    contradicting the fact that $\mu_r^{\learner, i+1}$ is closest to
    $\mu_r^{\teacher, i+1}$ among all points in
    $\Omega_r^\learner$. Denote by $\tilde \ell$ the line passing
    through $\tilde \mu_r$ and $\mu_r^{\learner, i}$.

    \begin{figure}[h]
        \centering
        \includegraphics[scale=1.1]{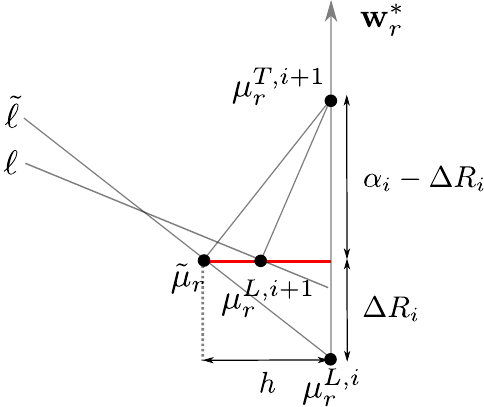}
        \caption{Illustration of the proof of Proposition
            \ref{prop:perf-increase}: The smaller the performance
            increase $\Delta R_i$, the better the upper bound on the
            gap $\overline R_{\Omega} - R(\mu_r^{\learner, i})$.}
        \label{fig:increase2}
    \end{figure}

    The facts that $\Omega_r^\learner$ is convex and that
    $\mu_r^{\learner, i+1} = \argmin_{\mu_r \in \Omega_r^{\learner}}
    \Vert \mu_r^{\teacher, i+1} - \mu_r\Vert_2$ imply that
    $\Omega_r^\learner$ must lie on one side of the hyperplane
    \begin{equation*}
        \mu_r^{\learner, i+1} + (\mu_r^{\teacher, i+1} -
        \mu_r^{\learner, i + 1})^\perp \subset \R^{d_r}.
    \end{equation*}
    Therefore, we can upper bound $\overline R_{\learner}$ in terms of
    the slope $s_\ell$ of the line $\ell$ which arises by
    intersecting that hyperplane with $V$:
    \begin{equation}
        \label{eq:bound-increase}
        \overline R_\learner \leq R(\mu_r^{\learner, i+1}) + D \cdot
        s_\ell
        = R(\mu_r^{\learner, i}) + \Delta R_i + D \cdot s_\ell.
    \end{equation}
    Note that the slope $s_\ell$ is upper bounded by the slope
    $s_{\tilde \ell}$ of $\tilde \ell$. We have
    $s_{\tilde \ell} = \frac{\Delta R_i}{h}$, where $h$ is the length
    of the red line segment in Figure \ref{fig:increase2}, and
    $h = \sqrt{(\alpha_i - \Delta R_i) \Delta R_i}$ by Pythagoras's
    theorem. Using that, we obtain
    \begin{equation}
        \label{eq:slope-bound}
        s_\ell \leq s_{\tilde \ell} = \sqrt{\frac{\Delta R_i}{\alpha_i - \Delta
                R_i}}.
    \end{equation}
    The claimed estimate \eqref{eq:gap-estimate} follows by plugging
    this upper bound for $s$ into \eqref{eq:bound-increase} and
    rearranging.
\end{proof}

\paragraph{Proof of Theorem \ref{thm:convergence}.}

\begin{proof}[Proof of Theorem \ref{thm:convergence}]
    The fact that $R(\mu_r^{\learner, i+1}) > R(\mu_r^{\learner, i})$,
    which is equivalent to $\Delta R_i > 0$, follows immediately from
    Proposition \ref{prop:perf-increase}.

    We now prove the claimed rate of convergence.

    First, using Proposition \ref{prop:perf-increase}, we note that
    the assumption that
    $\overline R_\learner - R(\mu_r^{\learner, i}) > \varepsilon$
    implies that
    \begin{equation}
        \label{eq:lower-bound}
        \varepsilon < \Delta R_i + D \sqrt{\frac{\Delta R_i}{\alpha_i
                - \Delta R_i}}.
    \end{equation}
    Using that, we can conclude that
    \begin{equation}
        \sqrt{\Delta R_i} > \min\{\sqrt{\varepsilon / 2}, \varepsilon
        \sqrt{\alpha_{\min}/(4D^2 + \varepsilon^2)}\}.\label{eq:lower-bound-min}
    \end{equation}
    Indeed, if $\Delta R_i \leq \frac{\varepsilon}{2}$, it follows
    from \eqref{eq:lower-bound} that we must have
    $D \cdot \sqrt{\Delta R_i/(\alpha_{\min} - \Delta R_i)} >
    \frac{\varepsilon}{2}$, which implies
    $\sqrt{ \Delta R_i} > \varepsilon \sqrt{\alpha_{\min} / (4D^2 +
        \varepsilon^2)}$. Since we are interested in the behavior as
    $\varepsilon \to 0$, we assume from now on that $\varepsilon$ is
    so small that
    $\varepsilon \sqrt{\alpha_{\min} / (4D^2 + \varepsilon^2)} <
    \sqrt{\varepsilon/2}$, so that \eqref{eq:lower-bound-min} becomes
    \begin{equation}
        \sqrt{\Delta R_i} > \varepsilon \sqrt{\alpha_{\min} / (4D^2 +
            \varepsilon^2)} =: C_0.\label{eq:lower-bound-C0}
    \end{equation}

    Second, we observe that 
    \begin{equation}
        \sqrt{\alpha_i - \Delta R_i} > \sqrt{\frac{\alpha_{\min}}{2}} =:
        C_1\label{eq:lower-bound-C1}
    \end{equation}
    except in at most
    $N := \frac{2}{\alpha_{\min}} (\max R\vert_\Omega - \min
    R\vert_\Omega)$ teaching steps. To see that, note that if the claimed
    inequality, which is equivalent to
    $\alpha_i - \frac{\alpha_{\min}}{2} > \Delta R_i$, does
    not hold, performance increases by at least
    $\Delta R_i \geq \frac{\alpha_{\min}}{2}$ as
    $\alpha_i > \alpha_{\min}$, and that can happen at most $N$ times.

    The inequalities \eqref{eq:lower-bound-C0} and
    \eqref{eq:lower-bound-C1} together imply that we have
    \begin{equation}
        C_0 \cdot C_1 \leq \sqrt{(\alpha_i - \Delta R_i) \Delta
            R_i}\label{eq:lower-bound-C0C1}
    \end{equation}
    as long as
    $\overline R_\learner - R(\mu_r^{\learner, i}) > \varepsilon$,
    except in at most $N$ teaching steps. Setting
    $C := \frac{1}{C_0 \cdot C_1}$, this is equivalent to
    \begin{equation}
        \label{eq:bound}
        \sqrt{\frac{\Delta R_i}{\alpha_i - \Delta R_i}} \leq C \Delta
        R_i
    \end{equation}
    Plugging \eqref{eq:bound} into the bound \eqref{eq:gap-estimate}
    provided by Proposition \ref{prop:perf-increase}, we obtain the
    estimate
    \begin{equation}
        \label{eq:bound-perf-increase}
        \frac{1}{1 + C D} (\overline R_\learner -
        R(\mu_r^{\learner, i})) \leq \Delta R_i.
    \end{equation}
    We have
    $C = \frac{1}{\varepsilon \alpha_{\min}} \sqrt{2(4D^2 +
        \varepsilon^2)}$, and hence
    \begin{equation}
        \label{eq:bound-constant}
        \frac{1}{1 + CD} = \frac{\varepsilon
            \alpha_{\min}}{\varepsilon \alpha_{\min} + \sqrt{2(4D^2 +
                \varepsilon^2)}\cdot D} \geq \frac{1}{1 + \sqrt{10}}
        \frac{\varepsilon \alpha_{\min}}{D^2} =: \lambda
    \end{equation}
    If we had the estimates \eqref{eq:bound-perf-increase},
    \eqref{eq:bound-constant} for \emph{all} teaching steps, we could
    conclude that the learner performance satisfies
    $R(\mu_r^{\learner, i}) > \overline R_\learner - 2\varepsilon$
    after at most
    $O(\frac{D^2}{\varepsilon \alpha_{\min}}\log
    \frac{D}{\varepsilon})$ teaching steps. One can see that e.g.\ by
    comparing the sequence $R_0, R_1, R_2,\dots$ with the solution
    $R(t)$ of the ordinary differential equation
    $\dot R = \lambda (\overline R_{\learner} - R)$, which satifies
    $\overline R_\learner - R(t) = (\overline R_\learner - R(0))
    \exp(-\lambda t)$. Since the number $N$ of teaching steps for which
    \eqref{eq:bound-perf-increase}, \eqref{eq:bound-constant} do potentially 
    \emph{not} hold is $O(\frac{D}{\alpha_{\min}})$, we can still make 
    this conclusion.
\end{proof}

\section{Background on (discounted) MCE-IRL Problem (Section~\ref{sec:learner})}\label{appendix:background-mceirl}


Our learner models build on the (discounted) Maximum Causal Entropy (MCE) IRL framework~\cite{ziebart2008maximum,ziebart2010modeling,ziebart2013principle,zhou2018mdce}. 
The results below are based on the MDCE-IRL formulation from  \cite{zhou2018mdce}.



\subsection{Primal problem}

In the standard (discounted) MCE-IRL framework, a learning agent aims to identify a policy that  matches the feature expectations of the teacher's demonstrations while simultaneously maximizing the (discounted) causal entropy of the policy, i.e., the learner solves the following  optimization problem:
\begin{align*}
\max_{\pi} \quad &H^\gamma(A_{0:\infty} \Vert S_{0:\infty}) := \sum_{t=0}^{\infty} \gamma^t \E\Big[-\log \pi(a_t \mid s_t) \Big] \\
\textnormal{subject to} \quad
&\mu_{r}(\pi)[i] = \hat{\mu}_{r}(\Xi^{\teacher})[i] \quad \forall i \in \{1, 2, \ldots, d_r\}.
\end{align*}
Here, $\mu_{r}(\pi)[i]$ and $\hat{\mu}_{r}(\Xi^{\teacher})[i]$ denote the scalar values of the $i^\textnormal{th}$ reward feature. The idea is that without any further information beyond the teacher's demonstrations, the most uncertain solution matching the reward feature expectation of those demonstrations should be preferred.

Formulating this as a minimization problem and spelling out all the constraints, we arrive at the following primal:
\begin{align*}
& \underset{\bm{\pi} = \{\pi_{t}\}_{t=0}^{\infty}}{\min} -H^\gamma(A_{0:\infty} \Vert S_{0:\infty})\\
\textnormal{subject to}&\\
& \mu_{r}(\pi_t)[i] = \hat{\mu}_{r}(\Xi^{\mathsf{T}})[i] \quad \forall i \in \{1, 2, \ldots, d_r\}\\
& \pi_{t}(a|s) \geq 0  \quad \forall a \in \mathcal{A}, s \in \mathcal{S}, t \geq 0\\
& \sum_{a \in \mathcal{A}} \pi_{t}(a|s) = 1  \quad \forall s \in \mathcal{S}, t \geq 0\\
& \pi_{t}(a|s) = \pi_{t'}(a|s) \quad \forall a \in \mathcal{A}, s \in \mathcal{S}, t \geq 0, t' \geq 0
\end{align*}
The last condition ensures that the policy $\pi$ is stationary.

\subsection{Lagrangian relaxation}\label{sec:stdmdce-lagrel}
The Lagrangian relaxation optimization formulation of the above primal problem is given by
\begin{align*}
\mathcal{L}(\bm{\pi}, \bm{\lambda}, \bm{\psi}) & = -H^\gamma(A_{0:\infty} \Vert S_{0:\infty}) + \bm{\lambda}^{\dagger}(\hat{\mu}_{r}(\Xi^{\mathsf{T}}) - \mu_{r}(\pi_t)) + \sum_{s,t} \psi_{s,t}(1 - \sum_{a \in \mathcal{A}} \pi_{t}(a|s))\\
\textnormal{subject to } &\\
&\pi_{t}(a|s) \geq 0 \quad  \forall a \in \mathcal{A}, s \in \mathcal{S}, t \geq 0\\
& \pi_{t}(a|s) = \pi_{t'}(a|s) \quad \forall a \in \mathcal{A}, s \in \mathcal{S}, t, t' \geq 0 
\end{align*}
Here, $\bm{\lambda} \in \mathbb{R}^{d_r}$ and $\bm{\psi} = \{\psi_{s,t}\}_{\forall s_t}$. Also, $\dagger$ is the transpose operator defined for vectors.\\

\emph{Remark.} The Lagrangian relaxation of the optimization problem is not convex in the problem variables because of the term $\bm{\lambda}^{\dagger}(\hat{\mu}_{r}(\Xi^\mathsf{T}) - \mu_{r}(\pi_{t}))$ in the objective function, which is not convex in the variables $\pi_{t}$. However, it can be shown that strong duality holds for both its dual and primal formulations (\cite{zhou2018mdce}). The dual formulation is described in Section \ref{sec:stdmdce-dual}.

\subsection{Parametric form of the policy}\label{sec:learner-stdpolicy}
For a given $\bm{\lambda}$, the optimal policy $\pi_{\bm{\lambda}}^{\soft}(a|s)$ is given by
\begin{align*}
\pi_{\bm{\lambda}}^{\soft}(a|s) & = \frac{\exp(Q_{\bm{\lambda}}^{\soft}(s,a))}{\exp(V_{\bm{\lambda}}^{\soft}(s))}
\end{align*}
where the quantities are defined recursively as follows:
\begin{align*}
Q_{\bm{\lambda}}^{\soft}(s,a) & = \bm{\lambda}^{\dagger} \mu_{r}({\pi_{\bm{\lambda}}^{\soft}}(a|s) )
+ \gamma  \sum_{s' \in \mathcal{S}} T(s'|s,a) V_{\bm{\lambda}}^{\soft}(s')\\
V_{\bm{\lambda}}^{\soft}(s) & = \text{ log } \sum_{a \in \mathcal{A}} \exp(Q_{\bm{\lambda}}^{\soft}(s,a) ) 
\end{align*}
This is shown by taking the derivative of the Lagrangian, $\mathcal{L}(\bm{\pi}, \bm{\lambda}, \bm{\psi})$ w.r.t.\ the primal variables $\pi_{t}$ and equating it to 0,  i.e.,
\begin{align*}
\frac{\partial L(\{\pi_{t}\}_{t=0}^{\infty}, \bm{\lambda}, \bm{\psi})}{\partial \pi_{t}} & = 0.
\end{align*}
For a given $\bm{\lambda}$, the corresponding softmax policy can be obtained by \emph{Soft-Value-Iteration} procedure (see  \cite[Algorithm.~9.1]{ziebart2010modeling}, \cite{zhou2018mdce}).

\subsection{Dual problem}\label{sec:stdmdce-dual}
For any given $\bm{\lambda}, \bm{\psi}$, let $g(\bm{\lambda}, \bm{\psi})$ be the optimal value for the optimization problem defined by the Lagrangian relaxation problem in Section \ref{sec:stdmdce-lagrel}. As \textit{strong duality} holds for the (discounted) MCE-IRL problem and its dual counter part, we solve only the following concave dual problem: 
\begin{align*}
&\underset{\bm{\lambda} \in \mathbb{R}^{d_r}, \psi_{s,t} \in \mathbb{R}}{\mathrm{maximize}} \quad g(\bm{\lambda}, \bm{\psi})
\end{align*}

\subsection{Gradients for the dual variables}
As the dual problem is concave, it can be solved using gradient ascent. The gradients of the dual function described in Section \ref{sec:stdmdce-dual} are given by:
\begin{align*}
\nabla_{\bm{\lambda}}\quad g & =  \hat{\mu}_{r}(\Xi^{\mathsf{T}})- \mu_{r}(\pi_{\bm{\lambda}}^{\soft})\\
\nabla_{\psi_{s,t}}\quad g & = 1 - \sum_{a \in \mathcal{A}} \pi_{\bm{\lambda}}^{\soft}(a|s)
\end{align*}
Here $\pi_{\bm{\lambda}}^{\soft}$ is the parametric softmax policy described above. The second condition is automatically satisfied because $\pi_{\bm{\lambda}}^{\soft}$ is a probability distribution.

The gradient update rule to compute the optimal $\bm{\lambda}$ is:
\begin{align*}
\bm{\lambda}_{\text{next}} \leftarrow \bm{\lambda} - \eta\cdot\big(\mu_{r}(\pi_{\bm{\lambda}}^{\soft}) - \hat{\mu}_{r}(\Xi^{\mathsf{T}})\big)
\end{align*}
where $\eta$ is the learning rate.

\section{Details of  (discounted) MCE-IRL Problem with Preferences  (Section~\ref{sec:learner:soft})}\label{appendix:mceirl-with-preferences}

Here we present the background of the learner model described in Section \ref{sec:learner:soft}. In this setting, the learner's preferences are modeled as linear soft constraints with L1 penalties. We consider the minimization variant of the problem. The results in this section follow directly from the analysis of Maximum Entropy Models under different constraints, as presented in \cite{tsujii2005maxent,schap2007maxent} when applied to (discounted) MCE-IRL problem \cite{ziebart2013principle,zhou2018mdce}. For brevity, redundant details of the derivations are omitted.\\
The final policy of the learner is given by $\pi_{\bm{\lambda}}^{\text{\soft}}$ and is defined in Section \ref{sec:learner-policy}.

\subsection{Primal problem}\label{sec:prefmdce-primal}
The primal problem is given by
\begin{align*}
& \underset{\bm{\pi} = \{\pi_t\}_{t=0}^{\infty};\ \delta^{\soft,\low}_{r},\ \delta^{\soft,\up}_{r} ,\ \delta^{\soft,\up}_{c} \geq 0}{\min} -H^{\gamma}(A_{0:\infty}|| S_{0:\infty}) + \sum_{i=1}^{d_r}C_r\cdot(\delta_{r}^{\soft,\low}[i] + \delta_{r}^{\soft,\up}[i]) + \sum_{j=1}^{d_c}C_c\cdot \delta_{c}^{\soft,\up}[j]\\
&\textnormal{subject to}\\
& \qquad \hat{\mu}_{r}(\Xi^{\mathsf{T}})[i] - \mu_{r}(\pi_t)[i] \leq \delta_{r}^{\soft,\low}[i] \quad \forall i \in \{1,2, \ldots, d_r\}\\
&\qquad \mu_{r}(\pi_t)[i] - \hat{\mu}_{r}(\Xi^{\mathsf{T}})[i]  \leq \delta_{r}^{\soft,\up}[i] \quad \forall i \in \{1,2, \ldots, d_r\}\\
& \qquad \mu_{c}(\pi_t)[j] \leq \delta_{c}^{\hard}[j] + \delta_{c}^{\soft,\up}[j] \quad  \forall j \in \{1,2, \ldots, d_c\}
\end{align*}
Here we have $\delta_{r}^{\soft,\low}, \delta_{r}^{\soft,\up} \in \mathbb{R}^{d_r}$ and $\delta_{c}^{\soft,\up} \in \mathbb{R}^{d_c}$ as the primal optimization slack variables with the constraint that $\delta^{\soft,\low}_{r}, \delta^{\soft,\up}_{r},\delta_{c}^{\soft,\up} \geq 0$. We also have $C_r > 0, C_c > 0$. $\delta_{c}^{\hard} \in \mathbb{R}^{d_c}$ is a given constant vector. 

\emph{Remark.}   \low~and \up~in the superscripts of dual variables represent whether they are variables for lower bound constraints or upper bound constraints.

\subsection{Lagrangian relaxation}\label{sec:prefmdce-lagrel}
The Lagrangian relaxation optimization formulation of the primal problem described in Section \ref{sec:prefmdce-primal} is given by
\begin{align*}
\mathcal{L}(\bm{\pi}, \delta^{\soft,\low}_{r},\delta^{\soft,\up}_{r},\delta_{c}^{\soft,\up},\bm{\lambda}, \bm{\psi}) & = -H^{\gamma}(A_{0:\infty}, S_{0:\infty}) + (\bm{\alpha}^{\low} - \bm{\alpha}^{\up})^{\dagger}(\hat{\mu}_{r}(\Xi^{\mathsf{T}}) - \mu_{r}(\pi_t))\\
& + \bm{\beta}^{\dagger}\mu_{c}(\pi_t)\\
& + \sum_{s,t} \psi_{s,t}(1 - \sum_{a \in \mathcal{A}} \pi_{t}(a|s)) 
-(\bm{\alpha}^{\low})^{\dagger} \delta^{\soft,\low}_{r} - (\bm{\alpha}^{\up})^{\dagger} \delta^{\soft,\up}_{r}\\ &-\bm{\beta}^{\dagger} \delta_{c}^{\soft,\up} -\bm{\beta}^{\dagger} \delta_{c}^{\hard} \\
& - (\bm{\rho}^{\low})^{\dagger} \delta_{r}^{\soft,\low} - (\bm{\rho}^{\up})^{\dagger} \delta_{r}^{\soft,\up}\\
& -\bm{\sigma}^{\dagger} \delta_{c}^{\soft,\up}\\
&+\sum_{i=1}^{d_r}C_r\cdot(\delta_{r}^{\soft,\low}[i] + \delta_{r}^{\soft,\up}[i]) + \sum_{j=1}^{d_c}C_c\cdot \delta_{c}^{\soft,\up}[j]\\
\textnormal{subject to } & \\
&\pi_{t}(a|s) \geq 0 \quad  \forall a \in \mathcal{A}, s \in \mathcal{S}, t \geq 0\\
& \pi_{t}(a|s) = \pi_{t^{'}}(a|s) \quad  \forall a \in \mathcal{A}, s \in \mathcal{S}, t, t^{'} \geq 0 
\end{align*}
Here, $\bm{\alpha}^{\low}, \bm{\alpha}^{\up}, \bm{\rho}^{low}, \bm{\rho}^{up} \in \mathbb{R}^{d_r}$, and $\bm{\beta}, \bm{\sigma} \in \mathbb{R}^{d_c}$. We also have non-negativity constraints on the dual variables: $\bm{\alpha}^{\low}, \bm{\alpha}^{\up}, \bm{\beta}, \bm{\rho}^{\low}, \bm{\rho}^{\up}, \bm{\sigma} \geq 0$. A few additional notes:
\begin{itemize}
	\item For convenience, we will denote the group of dual variables as $\bm{\lambda} := \{\bm{\alpha}^{\low}, \bm{\alpha}^{\up}, \bm{\beta}, \bm{\rho}^{\low}, \bm{\rho}^{\up}, \bm{\sigma}\}$
	\item The reward parameter $\bm{w_{\bm{\lambda}}} = [(\bm{\alpha}^{\low} - \bm{\alpha}^{\up})^{\dagger}, -\bm{\beta}^{\dagger}]^{\dagger}$ is used to define the learner's reward function
	$R_{\bm{\lambda}}(s) = \langle \bm{w}_{\bm{\lambda}}, \phi(s) \rangle$.
	\item $\dagger$ is the transpose operator, defined for vectors.
\end{itemize}

\subsection{Parametric form of the policy}\label{sec:learner-policy}
For a given, $\bm{\lambda} := \{ \bm{\alpha}^{\low}, \bm{\alpha}^{\up}, \bm{\beta}, \bm{\rho}^{\low}, \bm{\rho}^{\up}, \bm{\sigma} \}$, the optimal policy $\pi_{\bm{\lambda}}^{\soft}(a|s)$ is given by
\begin{align*}
\pi_{\bm{\lambda}}^{\soft}(a|s) & = \frac{\exp(Q_{\bm{\lambda}}^{\soft}(s,a))}{\exp(V_{\bm{\lambda}}^{\soft}(s))}
\end{align*}
where the quantities are defined recursively as follows:
\begin{align*}
Q_{\bm{\lambda}}^{\soft}(s,a) & = (\bm{\alpha}_{low} - \bm{\alpha}_{up})^{\dagger} \mu_{r}({\pi_{\bm{\lambda}}^{\soft}}(a|s) )
- \bm{\beta}^{\dagger} \mu_{c}({\pi_{\bm{\lambda}}^{\soft}}(a|s) )+ \gamma  \sum_{s^{'} \in \mathcal{S}} T(s^{'}|s,a) V_{\bm{\lambda}}^{\soft}(s^{'})\\
V_{\bm{\lambda}}^{\soft}(s) & = \text{ log }(\sum_{a \in \mathcal{A}} \exp(Q_{\bm{\lambda}}^{\soft}(s,a) ) )
\end{align*}
This is shown by taking the derivative of the Lagrangian, $\mathcal{L}(\bm{\pi}, \bm{\lambda}, \bm{\psi})$ w.r.t the primal variables, $\pi_{t}$ and equating it to 0. i.e.
\begin{align*}
\frac{\partial L(\{\pi_{t}\}_{t=0}^{\infty}, \bm{\lambda}, \bm{\psi})}{\partial \pi_{t}} & = 0
\end{align*}

\subsection{Updated Lagrangian}
We find the partial derivatives of the Lagrangian defined in Section \ref{sec:prefmdce-lagrel} w.r.t all the primal variables, 
$\delta^{\soft,\low}_{r}, \delta^{\soft,\up}_{r}, \delta_{c}^{\soft,\up}$:
\begin{align*}
\frac{\partial \mathcal{L}}{\partial \delta^{\soft,\low}_{r}[i]} & = 0\\
\Rightarrow \alpha^{\low}[i] & = C_r - \rho^{\low}[i]\\
\text{Also, } \frac{\partial \mathcal{L}}{\partial \delta^{\soft,\up}_{r}} & = 0\\
\Rightarrow \alpha^{\up}[i] & =C_r - \rho^{\up}[i]\\
\text{And, }  \frac{\partial \mathcal{L}}{\partial \delta^{\soft,\up}_{c}} & = 0\\
\Rightarrow \beta[i] & =C_r - \sigma[i]\\
\end{align*}
The dual variables satisfy $\bm{\sigma}, \bm{\rho}^{\low}, \bm{\rho}^{\up} \geq 0$. Hence, the above conditions translate into the following constraints on the set of dual variables, $\bm{\alpha}^{\low}, \bm{\alpha}^{\up}, \bm{\beta}$:
\begin{align*}
0 \leq {\alpha}^{\low}[i] \leq C_r \quad \forall i \in \{1,2, \ldots, d_r\} \\
0 \leq {\alpha}^{\up}[i] \leq C_r \quad \forall i \in \{1,2, \ldots, d_r\}\\
0 \leq {\beta}[j] \leq C_c \quad \forall j \in \{1,2, \ldots, d_c\}
\end{align*}
The updated Lagrangian now has these additional constraints and is given by:
\begin{align*}
\mathcal{L}(\bm{\pi}, \delta^{\soft,\low}_{r},\delta^{\soft,\up}_{r},\delta_{c}^{\soft,\up},\bm{\lambda}, \bm{\psi}) & = -H^{\gamma}(A_{0:\infty}, S_{0:\infty}) + (\bm{\alpha}^{\low} - \bm{\alpha}^{\up})^{\dagger}(\hat{\mu}_{r}(\Xi^{\mathsf{T}}) - \mu_{r}(\pi_t)) + \bm{\beta}^{\dagger}\mu_{c}(\pi_t)\\
& + \sum_{s,t} \psi_{s,t}(1 - \sum_{a \in \mathcal{A}} \pi_{t}(a|s)) 
-(\bm{\alpha}^{\low})^{\dagger} \delta^{\soft,\low}_{r} - (\bm{\alpha}^{up})^{\dagger} \delta^{\soft,\up}_{r}\\ &-\bm{\beta}^{\dagger} \delta_{c}^{\soft,\up} -\bm{\beta}^{\dagger} \delta_{c}^{\hard} \\
& - (\bm{\rho}^{\low})^{\dagger} \delta_{r}^{\soft,\low} - (\bm{\rho}^{\up})^{\dagger} \delta_{r}^{\soft,\up}\\
& -\bm{\sigma}^{\dagger} \delta_{c}^{\soft,\up}\\
&+\sum_{i=1}^{d_r}C_r\cdot(\delta_{r}^{\soft,\low}[i] + \delta_{r}^{\soft,\up}[i]) + \sum_{j=1}^{d_c}C_c\cdot \delta_{c}^{\soft,\up}[j]\\
\textnormal{subject to } &\\
&\pi_{t}(a|s) \geq 0 \quad  \forall a \in \mathcal{A}, s \in \mathcal{S}, t \geq 0\\
& \pi_{t}(a|s) = \pi_{t^{'}}(a|s) \quad  \forall a \in \mathcal{A}, s \in \mathcal{S}, t, t^{'} \geq 0 \\
&0 \leq {\alpha}^{\low}[i] \leq C_r \quad \forall i \in \{1,2, \ldots, d_r\}\\
&0 \leq {\alpha}^{\up}[i] \leq C_r \quad \forall i \in \{1,2, \ldots, d_r\}\\
&0 \leq {\beta}[j] \leq C_c \quad \forall j \in \{1,2, \ldots, d_c\}
\end{align*}
The set of dual variables becomes $\bm{\lambda} := \{\bm{\alpha}^{\low}, \bm{\alpha}^{\up}, \bm{\beta}\}$ and $\bm{\psi} = \{\psi_{s,t}\}_{\forall s_t}$.

\subsection{Dual problem}\label{appendix:mceirl-with-preferences-dualproblem}
For any given $\bm{\lambda, \psi}$, let $g(\bm{\lambda, \psi})$ be the optimal value for the Lagrangian relaxation problem. Strong Duality holds for both our primal and dual formulations, and the dual optimal policy is also optimal for the primal formulation. Hence, we solve the \textit{concave} dual problem, given by
\begin{align*}
&\underset{\bm{\alpha}^{\low},\bm{\alpha}^{\up} \in \mathbb{R}^{d_r}, \bm{\beta} \in \mathbb{R}^{d_c},  \psi_{s,t} \in \mathbb{R}}{\textnormal{maximize}} \text{ }g(\bm{\lambda}, \bm{\psi})\\
\textnormal{subject to } &\\
&\quad 0 \leq \bm{\alpha}^{\low} \leq C_r \\
&\quad 0 \leq \bm{\alpha}^{\up} \leq C_r \\
&\quad 0\leq \bm{\beta} \leq C_c \\
\end{align*}
where $\bm{\lambda} := \{\bm{\alpha}^{\low}, \bm{\alpha}^{\up}, \bm{\beta}\}$.

\subsection{Gradients for the dual problem}
 As the dual problem is concave, it can be solved using gradient ascent.\\
 Note that,
\begin{align*}
\nabla_{\psi_{s,t}}g & = 1 - \sum_{a \in \mathcal{A}} \pi_{\bm{\lambda}}^{\soft}(a|s)
\end{align*}
Here $\pi_{\bm{\lambda}}^{\soft}$ is the parametric softmax policy described above. This condition is automatically satisfied because $\pi_{\bm{\lambda}}^{\soft}$ is a probability distribution.
For the remaining dual variables, we have the following gradients:
\begin{align*}
\nabla_{\bm{\alpha}^{\low}}\text{ }g & = \hat{\mu}_{r}(\Xi^{\mathsf{T}})-\mu_{r}(\pi_{\bm{\lambda}}^{\soft})\\
\nabla_{\bm{\alpha}^{\up}}\quad g & =  \mu_{r}(\pi_{\bm{\lambda}}^{\soft})-\hat{\mu}_{r}(\Xi^{\mathsf{T}})\\
\nabla_{\bm{\beta}}\text{ }g & = \mu_{c}(\pi_{\bm{\lambda}}^{\soft})
\end{align*}

The (projected) gradient update rules to compute the optimal value of the dual variables $(\bm{\alpha}^{\low}, \bm{\alpha}^{\up}, \bm{\beta})$ are given by the following:
\begin{align*}
\bm{\alpha}^{\low}_{\text{next}} & \leftarrow \bm{\alpha}^{\low} - \eta\cdot(\mu_{r}(\pi_{\bm{\lambda}}^{\soft})-\hat{\mu}_{r}(\Xi^{\mathsf{T}}))\\
{\alpha}^{\low}_{\text{next}}[i] & \leftarrow \max (0, {\alpha}^{\low}_{\text{next}}[i]) \quad \forall i \in \{1,2, \ldots, d_r\}\\
{\alpha}^{\low}_{\text{next}}[i] & \leftarrow \min (C_r, {\alpha}^{\low}_{\text{next}}[i]) \quad \forall i \in \{1,2, \ldots, d_r\}
\end{align*}
\begin{align*}
\bm{\alpha}^{\up}_{\text{next}} & \leftarrow \bm{\alpha}^{\up} - \eta\cdot(\hat{\mu}_{r}(\Xi^{\mathsf{T}}) - \mu_{r}(\pi_{\bm{\lambda}}^{\soft}))\\
{\alpha}^{\up}_{\text{next}}[i] & \leftarrow \max (0, {\alpha}^{\up}_{\text{next}}[i]) \quad \forall i \in \{1,2, \ldots, d_r\}\\
{\alpha}^{\up}_{\text{next}}[i] & \leftarrow \min (C_r, {\alpha}^{\up}_{\text{next}}[i])\quad \forall i \in \{1,2, \ldots, d_r\}
\end{align*}
\begin{align*}
\bm{\beta}_{\text{next}} & \leftarrow \bm{\beta} - \eta\cdot(-\mu_{c}(\pi_{\bm{\lambda}}^{\soft}))\\
{\beta}_{\text{next}}[j] & \leftarrow \max (0, {\beta}_{\text{next}}[j]) \quad \forall j \in \{1,2, \ldots, d_c\}\\
{\beta}_{\text{next}[j]} & \leftarrow \min (C_{c}, {\beta}_{\text{next}}[j]) \quad \forall j \in \{1,2, \ldots, d_c\}
\end{align*}
where $\eta$ is the learning rate.

\section{LP Formulation for the Teacher \algAwareCMDP (Section~\ref{sec:teacher.fullknowledge.special})}\label{appendix:lp}

The problem of finding optimal learner-aware teaching demonstrations for the learner in Section~\ref{sec:learner:hard} with linear preferences can be formulated as the following linear program (based on the linear programming formulation for solving MDPs~\cite{de1960problemes}):
  \begin{align}  
     \max_{z} \quad & \sum_s \sum_a z(s,a) \langle \wopt, \phi_r(s) \rangle \\
     \textnormal{s.t.} \quad & 
      \sum_a z(s',a) = (1-\gamma) P_0(s') + \gamma \sum_s \sum_a T(s' | s, a) z(s, a) \quad \forall s' \\
      &z(s,a) \geq 0 \quad \forall s,a \\
      & \sum_s \sum_a z(s,a) \phi_c(s)[j] \leq \delta^{\hard}_c[j] \quad \forall j \in \{ 1, 2, \ldots, d_c \} \label{eq:pref-constr-lp}
  \end{align}
Here $z$ is a vector of discounted state-action frequencies and $z(s,a)$ refers to state-action frequency for state $s$ and action $a$.
The constraints in~\eqref{eq:pref-constr-lp} are the linear preference constraints.
From the optimal solution of the LP, an optimal stochastic policy can be extracted by
\begin{align}
      \pi(s, a) := \frac{z(s,a)}{\sum_{a'} z(s,a')}.
\end{align}

\label{appendix:bi-level}

\section{Bi-Level Optimization Approach (Section \ref{sec:teacher.fullknowledge.generic})}
\label{appendix:bi-level}

We only show the formalism for the most general bi-level problem for learners with linear preferences.


\subsection{Using Dual (discounted) MCE-IRL formulation for the learner model in Section~\ref{sec:learner:soft}}

The basic bi-level optimization problem that we aim to solve is the following:
\begin{align*}
    \max_{\pi^\teacher} &\quad R(\pi^\learner) \\
       \textnormal{subject to} &\quad \pi^\learner \in \arg \max_{\pi} \textnormal{IRL}(\pi,\mu(\pi^\teacher)).
\end{align*}

We will replace the lower-level problem, i.e., $\arg \max_{\pi} \textnormal{IRL}(\pi,\mu(\pi^\teacher))$ with its Karush-Kuhn-Tucker conditions \cite{boyd2004convex,sinha2018review}.  The lower-level problem in its dual formulation is given in Appendix~\ref{appendix:mceirl-with-preferences-dualproblem}.

Omitting details and replacing $R(\pi_{\bm{\lambda}}) := \langle \wopt, \mu_r(\pi_{\bm{\lambda}})\rangle$, this yields problems of  the following form:
\begin{align*}
    \max_{\bm{\lambda}} &\quad \langle \wopt, \mu_r(\pi_{\bm{\lambda}})\rangle  \\
       \textnormal{subject to:}&\\
            &\quad 0 \leq \bm{\alpha}^{\low} \leq C_r \\
             &\quad 0 \leq \bm{\alpha}^{\up} \leq C_r \\
            &\quad 0\leq \bm{\beta} \leq C_c \\
            &\quad \mu_c(\pi_{\bm{\lambda}}) \leq (\geq) \delta^{\hard}_c
\end{align*}
where $\bm{\lambda} := \{\bm{\alpha}^{\low}, \bm{\alpha}^{\up}, \bm{\beta}\}$.  Here $\pi_{\bm{\lambda}}$ corresponds to a \emph{softmax} policy with a reward function $R_{\bm{\lambda}}(s) = \langle \bm{w}_{\bm{\lambda}}, \phi(s) \rangle$ for  $\bm{w}_{\bm{\lambda}} = [(\bm{\alpha}^{\low}-\bm{\alpha}^{\up})^\dagger, -\bm{\beta}^\dagger]^\dagger$. Thus, finding optimal demonstrations means optimization over \emph{softmax} teaching policies while respecting the learner's preferences.

\subsubsection{Optimal solution}
The cases of the above problem we can observe have to be solved separately and the best solution must be picked. That is, we 
find the following two solutions: (step i) $\bm{\lambda}^*_1$, and (step ii) $\bm{\lambda}^*_2$.  Then pick the best $\bm{\lambda}^*$ in (step iii):\\\\
\textbf{Step i: $\bm{\lambda}^*_1$}
Compute optimal parameters $\bm{\lambda}^*_1$ by solving the following problem:
\begin{align*}
    \max_{\lambda} &\quad  \langle \wopt, \mu_r(\pi_{\bm{\lambda}})\rangle \\
       \textnormal{subject to:}&\\
            &\quad 0 \leq \bm{\alpha}^{\low} \leq C_r \\
             &\quad 0 \leq \bm{\alpha}^{\up} \leq C_r \\
            &\quad 0\leq \bm{\beta} \leq C_c \\
            &\quad \mu_c(\pi_{\bm{\lambda}}) \leq \delta^{\hard}_c
\end{align*}
\\\\
\textbf{Step ii: $\bm{\lambda}^*_2$}
Compute optimal parameters $\bm{\lambda}^*_2$ by solving the following problem:
\begin{align}
    \max_{\bm{\lambda}} &\quad \langle \wopt, \mu_r(\pi_{\bm{\lambda}}) \rangle  \\
       \textnormal{subject to:}&\\
            &\quad 0 \leq \bm{\alpha}^{\low} \leq C_r \\
             &\quad 0 \leq \bm{\alpha}^{\up} \leq C_r \\
            &\quad \bm{\beta} = C_c \\
            &\quad \mu_c(\pi_{\bm{\lambda}}) \geq \delta^{\hard}_c
\end{align}
\textbf{Step iii: $\bm{\lambda}^*$}
Pick the best solution as
\begin{align*}
    \bm{\lambda}^* = \arg \max_{\bm{\lambda} \in \{\bm{\lambda}^*_1, \bm{\lambda}^*_2 \}} \langle \wopt, \mu_r(\pi_{\bm{\lambda}}) \rangle
\end{align*}
This provides the optimal policy for the teacher. The teacher then computes feature expectation of this policy and provide it to the learner.

\subsection{Solving the above problem}
\label{sec:bi-level-solving}

We adopt a variant of the Frank-Wolfe algorithm~\cite{jaggi2013} to solve the problems of the form:
\begin{align}
    \max_{\bm{\lambda}} &\quad R(\pi_{\bm{\lambda}}) := \langle \wopt, \mu_r(\pi_{\bm{\lambda}}) \rangle \label{eq:objective-bilevel}\\
       \textnormal{subject to:}&\\
            &\quad 0 \leq \bm{\alpha}^{\low} \leq C_r \\
             &\quad 0 \leq \bm{\alpha}^{\up} \leq C_r \\
            &\quad 0\leq \bm{\beta} \leq C_c \\
            &\quad \mu_c(\pi_{\bm{\lambda}}) \leq (\geq) \delta^{\hard}_c
\end{align}
In particular, we take the following steps to optimize the teaching policy $\pi_{\bm{\lambda}}$:
\begin{enumerate}
    \item \emph{Initialization.} Find a feasible starting point $\bm{\lambda}_0$
    \item \emph{Optimization.} For $t=1, 2, \ldots$
      \begin{itemize}
          \item Compute the gradient $\bm{g}_t = [\nabla_{\bm{\lambda}} R(\pi_{\bm{\lambda}})](\bm{\lambda}_{t-1})$ of the objective at $\bm{\lambda}_{t-1}$. In experiments we approximate the gradient using finite-differences.
          \item Linearize the constraints $\mu_c(\pi_{\bm{\lambda}}) \leq (\geq) \delta^{\hard}_c$ at $\bm{\lambda}_{t-1}$ as $\bm{b}_{t} + \bm{A}_{t} (\bm{\lambda} - \bm{\lambda}_{t-1}) \leq (\geq) \delta^{\hard}_c$, where $\bm{b}_{t}=\mu_c(\pi_{\bm{\lambda}_{t-1}})$ and $\bm{A}_t = [\nabla_{\bm{\lambda}} \mu_c(\pi_{\bm{\lambda}})](\bm{\lambda}_{t-1})$. Again, we employ finite-differences to approximate this linearization. Clearly, we can reuse computation from the gradient estimation of the objective here to reduce computational demands.
          \item Solve the direction-finding subproblem (a linear problem):
            \begin{align*}
    \max_{\bm{\gamma}} &\quad \langle \bm{\gamma}, \bm{g}_t \rangle \label{eq:objective-bilevel}\\
       \textnormal{subject to:}&\\
            &\quad 0 \leq \bm{\alpha}^{\low} \leq C_r \\
             &\quad 0 \leq \bm{\alpha}^{\up} \leq C_r \\
            &\quad 0\leq \bm{\beta} \leq C_c \\
            &\quad \bm{b}_{t} + \bm{A}_{t-1} (\bm{\lambda} - \bm{\lambda}_{t-1}) \leq (\geq) \delta^{\hard}_c
            \end{align*}
         with optimal solution $\bm{\gamma}_t^*$. Assuming that the linear approximation of the constraints is accurate locally, the directional vector $\bm{d}_t = \bm{\gamma}_t^* - \bm{\lambda}_{t-1}$ is an ascent direction.
         
     \item Perform a line-search from $\bm{\lambda}_{t-1}$ to $\bm{\gamma}_t^*$ and let $\bm{\lambda}_t$ be the point that maximizes the line search.
     
     \item Upon convergence, terminate the For loop.
    \end{itemize}
\end{enumerate}

Upon convergence of the algorithm, the teacher can use the final $\bm{\lambda}_t$ for teaching.

\emph{Remark.} Observe that the above algorithm would reduce to the standard Frank-Wolfe algorithm with line-search in the case of linear inequalities only.

}
}
{}
\end{document}